\documentclass[10pt,twocolumn,letterpaper]{article}

\usepackage[arxiv]{iccv}
\usepackage{soul}

\definecolor{iccvblue}{rgb}{0.21,0.49,0.74}
\usepackage[pagebackref,breaklinks,colorlinks,allcolors=iccvblue]{hyperref}
\usepackage[framemethod=tikz]{mdframed}
\usepackage[most]{tcolorbox}
\usepackage{float}

\newmdenv[innerlinewidth=1.0pt,
linecolor=black,
backgroundcolor=cyan!60!black!5,
roundcorner=4pt,
innerleftmargin=10pt, innerrightmargin=10pt,
innertopmargin=6pt,innerbottommargin=6pt]{goal}

\newenvironment{summary}{
\begin{tcolorbox}[
    enhanced,
    attach boxed title to top center={yshift=-3mm,yshifttext=-1mm},
    colframe=green!45!black,  
    colback=green!5!white!35,    
    colbacktitle=green!40!black!60, 
    title={Section Summary},
    fonttitle=\bfseries,
    boxed title style={size=small,colframe=green!50!black},
    beforeafter skip=2pt,
    boxrule=0.75pt
]
}
{\end{tcolorbox}}

\usepackage[disable]{todonotes}
\setuptodonotes{color=gray!10, size=scriptsize}
\DeclareMathOperator{\Tr}{Tr}
\DeclareMathOperator{\blkdiag}{blkdiag}

\def\paperID{15322} 
\def\confName{ICCV}
\def\confYear{2025}

\title{
    From Linearity to Non-Linearity: \\
    How Masked  Autoencoders Capture Spatial Correlations}

\author{Anthony Bisulco\thanks{\scriptsize Equal Contribution}\\
University of Pennsylvania\\
3317 Chestnut St, Philadelphia PA\\
{\tt\small abisulco@seas.upenn.edu}
\and
Rahul Ramesh*\\
University of Pennsylvania\\
3317 Chestnut St, Philadelphia PA\\
{\tt\small rahulram@seas.upenn.edu}
\and
Randall Balestriero\\
Brown University\\
115 Waterman St, Providence, RI\\
{\tt\small randall\_balestriero@brown.edu}
\and
Pratik Chaudhari\\
University of Pennsylvania\\
3317 Chestnut St, Philadelphia PA\\
{\tt\small pratikac@seas.upenn.edu}
}

\begin{document}
\maketitle
\setcounter{page}{1}
\makeatletter{\renewcommand*{\@makefnmark}{}
\footnotetext{\scriptsize AB supported by NSF FRR 2220868, NSF IIS-RI 2212433,   ARO MURI W911NF-20-1-0080, and ONR N00014-22-1-2677, and AB, RR, and PC supported by NSF IIS-2145164, CCF-2212519, ONR N00014-22-1-2255, and DSO Laboratories, Singapore.}\makeatother}

\begin{abstract}
    Masked Autoencoders (MAEs) have emerged as a powerful pretraining technique for vision foundation models.
    Despite their effectiveness, they require extensive hyperparameter tuning (masking ratio, patch size, encoder/decoder layers) when applied to novel datasets.
    While prior theoretical works have analyzed MAEs in terms of their attention patterns and hierarchical latent variable models, the connection between MAE hyperparameters and performance on downstream tasks is relatively unexplored.
    This work investigates how MAEs learn spatial correlations in the input image.
    We analytically derive the features learned by a linear MAE and show that masking ratio and patch size can be used to select for features that capture short- and long-range spatial correlations.
    We extend this analysis to non-linear MAEs to show that MAE representations adapt to spatial correlations in the dataset, beyond second-order statistics.
    Finally, we discuss some insights on how to select MAE hyper-parameters in practice.
\end{abstract}   

\begin{figure}
    \centering
    \includegraphics[width=\linewidth]{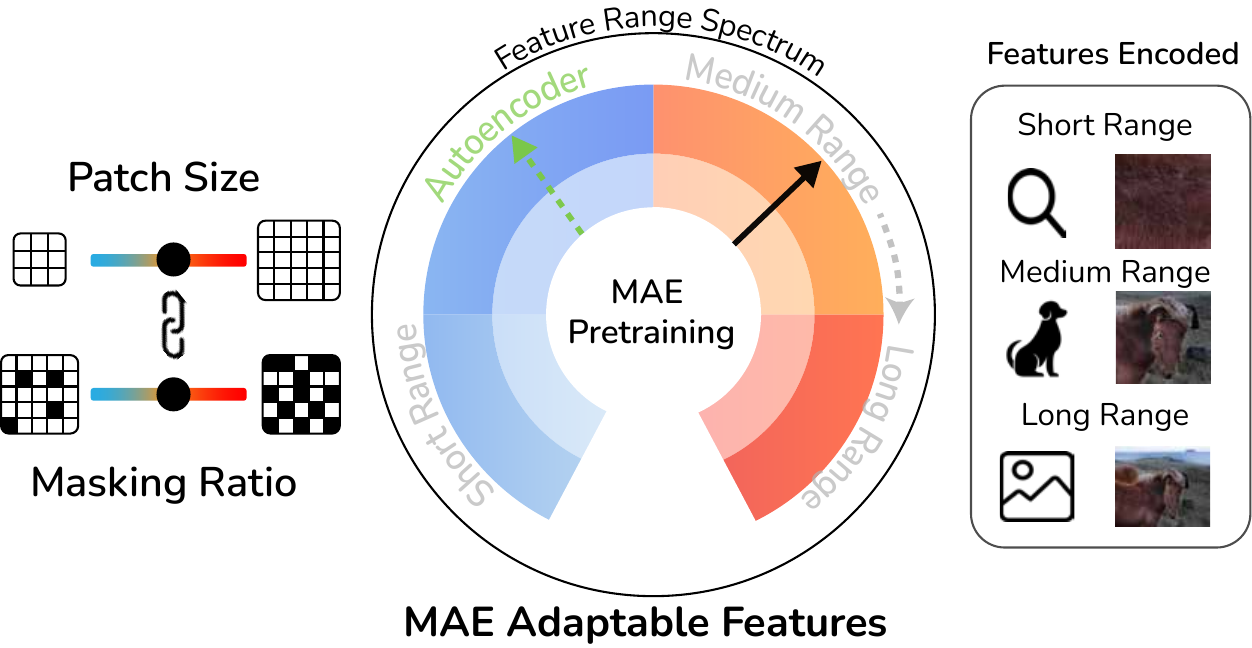}
    \caption{ MAEs~(Black Arrow) can adapt the spatial scale of features they encode through two key hyperparameters: patch size and masking ratio. Generally, increasing either of these hyperparameters yields features with a broader spatial extent. Selecting these hyperparameters allows MAEs to be used for datasets with different spatial scales. In contrast, standard autoencoders~(Green Arrow) learn short-range spatial features. On the right side of the figure, we show an example of feature spatial scales ranging from local textures to objects to global scenes. }
    \label{fig:task-scale}
    \vspace{-2mm}
\end{figure}
\vspace{-4mm}
\section{Introduction}

\todo{Make sure all figures are referenced}
\todo{Reference all content in the appendix}
\todo{Make sure results / section 2 discusses how masking ratio and patchs-size control spatial correlations}


Masked autoencoders (MAEs) are a powerful pretraining technique for images~\cite{He2021MAE, Xie2021SimMIMAS} and video~\cite{feichtenhofer2022masked, tong2022videomae, Wang2023VideoMAEVS}  
playing a key role in pretraining for foundational models such as Segment Anything~\cite{Kirillov2023SegmentA}, InternVideo~\cite{wang2024internvideo}, EVA~\cite{fang2023eva}, and Unified IO~\cite{lu2022unified}.  Despite their success, MAEs are not yet well understood---key design choices such as masking ratio, patch size, and model depth heavily rely on heuristics. In this work, we aim to better understand MAEs by analyzing their underlying mechanisms and their role in representation learning.

A key challenge in applying MAEs to new datasets or modalities is the need for extensive cross-validation to tune hyperparameters such as patch size and masking ratio. Suboptimal choices can lead to degraded task performance and unnecessary computational cost. Exhaustive search over these parameters is often impractical, as MAEs’ long training times make full cross-validation prohibitively expensive for large datasets or high-dimensional inputs. As a step towards addressing this challenge, we aim to study how masking ratio and patch size affect the representations learned by MAEs (\cref{fig:task-scale}).

Our central hypothesis is that \textit{MAEs learn spatial correlations in the input image}, with the masking ratio and patch size controlling the spatial scale of these correlations. We argue that MAEs introduce a spatial locality bias in ViT (Vision Transformer) architectures that inherently lack it, although some efforts have attempted to explicitly incorporate such biases~\citep{liu2021swin,d2021convit}. We make the following key contributions:
\begin{enumerate}
    \item We derive an analytical expression for the encoder and decoder of a linear MAE. 
    We show how MAEs can learn features that capture short- or long- range correlations by controlling the masking ratio and patch size.
    \item We show that MAEs can be viewed as learning an adaptive basis for denoising, progressively shifting from localized to global feature representations over the course of training.
    \item We provide insights for practitioners on how to select hyperparameters such as masking ratio, patch size, and number of encoder/decoder layers in MAEs. 
\end{enumerate}
\section{Linear Masked Autoencoders}\label{s:linear_mae}

Reconstruction is a popular objective for self-supervised learning, but it often fails to learn useful features for downstream perception tasks~\citep{balestriero2024learning}.
Masked reconstruction~\citep{He2021MAE,Xie2021SimMIMAS} has emerged as a popular choice for pretraining deep networks.
In this section, we study how reconstruction differs from masked reconstruction and why the latter is useful for downstream perception tasks.

\begin{figure}[!htpb]
    \centering
    \includegraphics[width=0.5\linewidth]{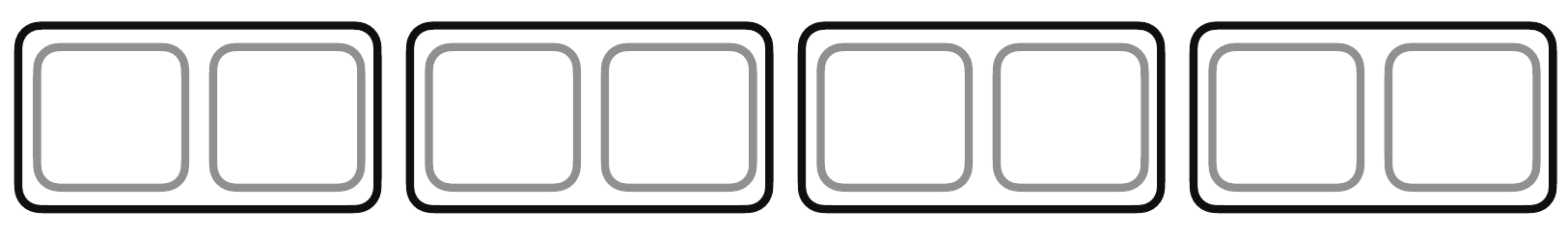}
    \caption{A $d=8$ dimensional input can be divided up into 4 patches, each of size $p=2$.}
    \label{fig:mask_illustrate}
\end{figure}

We first consider a linear MAE and analyze the solutions obtained for different masking ratios and patch sizes.
Consider input data $X \in \mathbb{R}^{n \times d}$ where $n$ is the number of samples, and each sample is a real-valued vector in $d$ dimensions.
Consider a linear encoder with weights $A \in \mathbb{R}^{d \times k}$ and a linear decoder with weights $B\in \mathbb{R}^{k \times d}$---the network has no non-linearities. The encoder projects the input from $d$ input dimensions to $k$ latent dimensions, and the decoder projects the representation back to $d$ dimensions.
A linear MAE minimizes the objective
\begin{equation}\label{eq:mae_objective}
    \ell_m(A,B) = \mathbb{E}_{R} \left[\left\| X - (R \odot X)AB \right\|^2 \right],
\end{equation}
where $R \in \{0,1\}^{n \times d}$ is a random variable denoting a mask and $\odot$ is the element-wise product.
We can arrange our $d$-dimensional input vector into patches, each of size $p$ (see~\cref{fig:mask_illustrate}).
This is similar to how images are divided into patches in ViTs~\cite{dosovitskiy2020image}. Patches are either fully masked or fully visible; this means $R_{ki} = R_{kj}$ if $i$ and $j$ belong to the same patch for any image with index $k$.
Each patch is masked according to a Bernoulli random variable with probability $m$.%
\footnote{Typical implementations of MAEs exactly mask a fraction $m$ of the patches instead of using a random variable to determine if a patch should be masked. Typically, MAEs also reconstruct only the unmasked patches at the output of the decoder, instead of employing an objective that reconstructs all patches, as we have done here.
}

\cref{lemma:mae_expectation} computes the expectation in ~\cref{eq:mae_objective} in closed form to reduce the linear MAE objective to
\begin{equation}\label{eq:mae_marginal}
    \hspace{-3mm} \ell_m(A,B) {=} \underbrace{\left\| X {-} (1{-}m) XAB \right\|^2}_{\text{reconstruction}} \hspace{-1mm}{+} m(1{-}m) \underbrace{\left\| GAB \right\|^2}_{\text{regularizer}}\hspace{-4mm}
\end{equation}
where $G \in \mathbb{R}^{d \times d}$ satisfies $G^\top G = \blkdiag_p(X^\top X)$. The block-diagonal matrix $\blkdiag_p(X^\top X)$ consists of block diagonal entries of size $p \times p$, with the off-diagonal blocks set to zero. The MAE objective in~\cref{eq:mae_marginal} consists of two loss terms: the first term is like the reconstruction loss of an autoencoder that encourages $X \approx (1-m) X A B$, the second term $\left\| G A B\right\|^2$ acts as a regularizer tuned to the data via $G$. The masking ratio $m$ controls the strength of the regularizer while the patch size controls the block-diagonal structure of the regularizer. 
Note that $m=0$ recovers the objective for a non-masked autoencoder (AE). The regularizer term here, therefore, is the ``bias'' of an MAE. It forces the representation of an MAE to deviate from that of an AE.

\subsection{Characterizing the minima of MAEs}\label{ss:minima}

We next analyze the encoder obtained in a linear MAE using~\cref{eq:mae_marginal}. The classic work of \citet{baldi1989neural} shows that AEs are closely related to principal component analysis (PCA)~\citep{baldi1989neural}. AE's encoder at any global minimum extracts the top-$k$ principal components of the data. Linear AEs, therefore, discover features that best explain the variance in the input data.
The training objective of a linear AE does not have local minima. All global minima are of the form $ A = U C^{-1}, \quad B = C U^\top$, where the columns of $U \in \mathbb{R}^{d \times k}$ are the top-$k$ eigenvectors of $X^\top X$, and $C \in \mathbb{R}^{k \times k}$ is any invertible matrix.

We can perform a similar analysis for linear MAEs to derive an analytical expression for the optimal encoder and decoder and characterize the set of all critical points of this objective. Note that while there are no non-linearities in a linear MAE, the training objective is still non-convex.
\begin{theorem}\label{thm:mae_minima}
    Let $V = (1-m)X^\top X + m \blkdiag_p(X^\top X)$ and let the columns of $U_k$ denote the top-k eigenvectors of $X^\top X V^{-1} X^\top X$. The objective in~\cref{eq:mae_objective} is minimized when the decoder 
\begin{equation}
    B = C U_k 
\end{equation}
and the encoder
\begin{equation}
    A = V^{-1} X^\top X B^\top (BB^\top)^{-1} C^{-1},
    \label{eq:A_opt}
\end{equation}
where $C$ is any invertible matrix of size $k \times k$.
\end{theorem}

Every critical point of the MAE objective is a subset of $k$ eigenvectors of $X^\top X V^{-1} X^\top X$ (from~\cref{lemma:generalized_eigenvalue}), i.e., there are exponentially many critical points.
The product of the encoder and decoder for an MAE comprises two parts. The first part $V^{-1} X^\top X$ whitens the data using a weighted mixture of $X^\top X$ and $\blkdiag(X^\top X)$. The second part $B^\top (BB^\top)^{-1} B$, projects the data onto the column space of $B$.

\begin{remark}
    If the input correlation matrix $X^\top X$ is itself block diagonal, i.e., spatial correlations in the input data are non-zero within a patch of size $p$, then $V = (1-m)X^\top X + m \blkdiag_p(X^\top X) = X^\top X$, which recovers the regular autoencoder solution from~\cref{thm:mae_minima}, with a patch size of $p$. In this case, a linear autoencoder and a masked autoencoder have identical minima.
\end{remark}

\paragraph{Understanding the differences between an MAE and an AE using an Ising model.}
\begin{figure}[!htpb]
    \centering
    \includegraphics[width=0.60\linewidth]{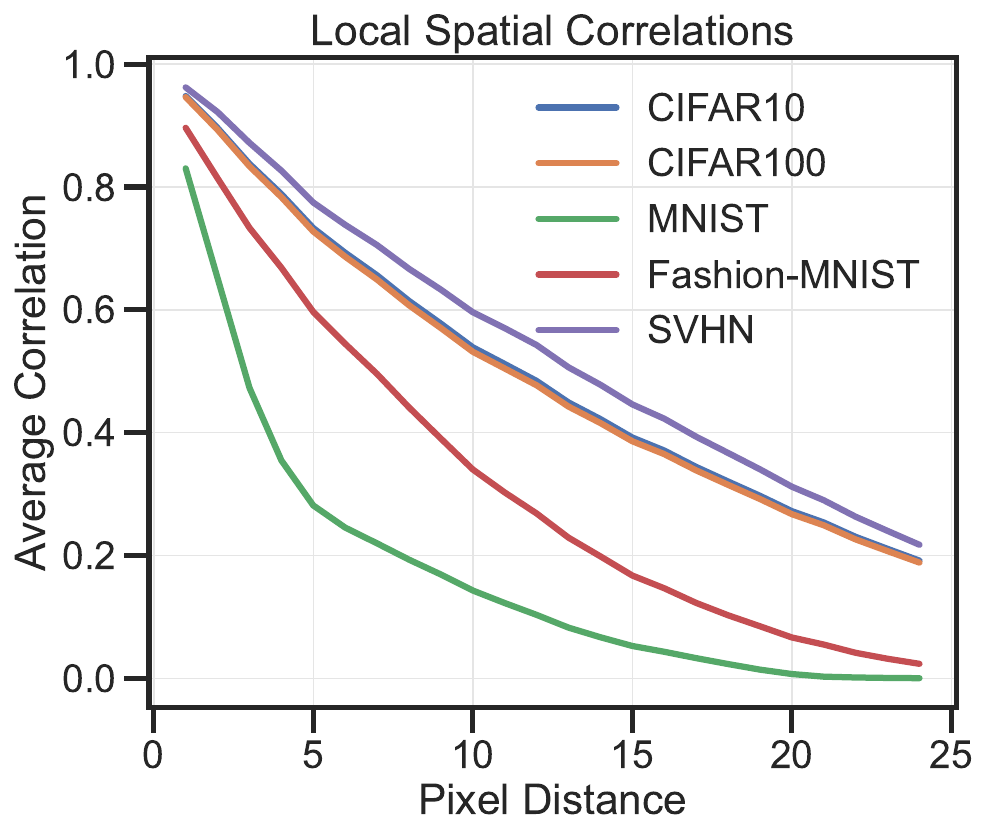}
    \caption{Spatial correlations between pixels as a function of the distance between them for different datasets. Correlation between pixels is inversely correlated to the distance for all 5 datasets, i.e., images have stronger local correlations.}
    \label{fig:local_correlations}
\end{figure}
Images exhibit strong local correlations (\cref{fig:local_correlations}) and the strength of the correlations between pixels is inversely proportional to the distance between them~\cite{natural_stats}.  This structure can be quantified using the spatial autocorrelation:
\begin{equation}
    \resizebox{.90\linewidth}{!}{$
    R(\Delta x, \Delta y) = \frac{1}{MN}\sum_x\sum_y f(x,y)f(x+\Delta x, y+\Delta y)
    $}
\end{equation}
which measures how a function $f(x,y)$ pixel values correlate across spatial distances $(\Delta x, \Delta y)$, for a 2D pixel grid of size $(M,N)$.
MAEs can exploit these inherent correlations in natural image data for masked patch reconstruction. To emulate a part of this structure, we consider data drawn from an Ising model. This model exhibits strong local correlations controlled by the coupling constant $J$. Under the Ising model, the input $x \in \{-1, 1\}^d$ has probability 
\begin{equation}
    p(x) \propto \exp \left( J x_1 x_d +  J \sum_{i=1}^{d-1} x_i x_{i+1} \right).
\end{equation}
We fit a linear AE and a linear MAE on samples from this distribution using the objective in~\cref{eq:mae_objective} for the MAE.
We approximate the correlation between $x_i$ and $x_j$ by $\langle x_i, x_j \rangle = \tanh(J)^r$, which is its value in the asymptotic limit of $d \rightarrow \infty$, where $r = \min (|i-j|, d - |i-j|)$ is the distance between $x_i$ and $x_j$~\citep{huang2009introduction}. 

We use this to approximate $X^\top X$, which can be substituted in the analytical expressions in~\cref{thm:mae_minima}.

\begin{figure}[ht]
    \centering
    \includegraphics[width=0.48\linewidth]{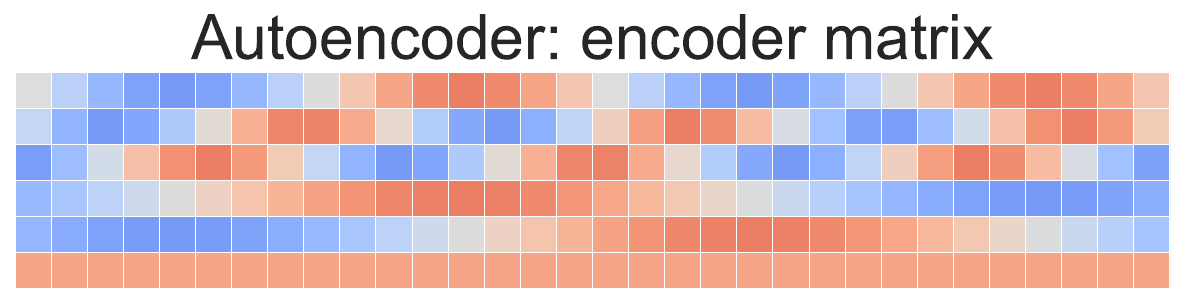}
    \includegraphics[width=0.48\linewidth]{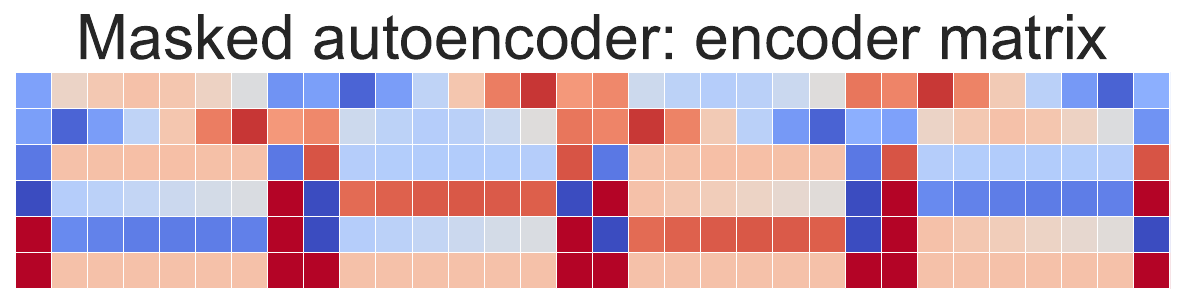}
    \includegraphics[width=0.48\linewidth]{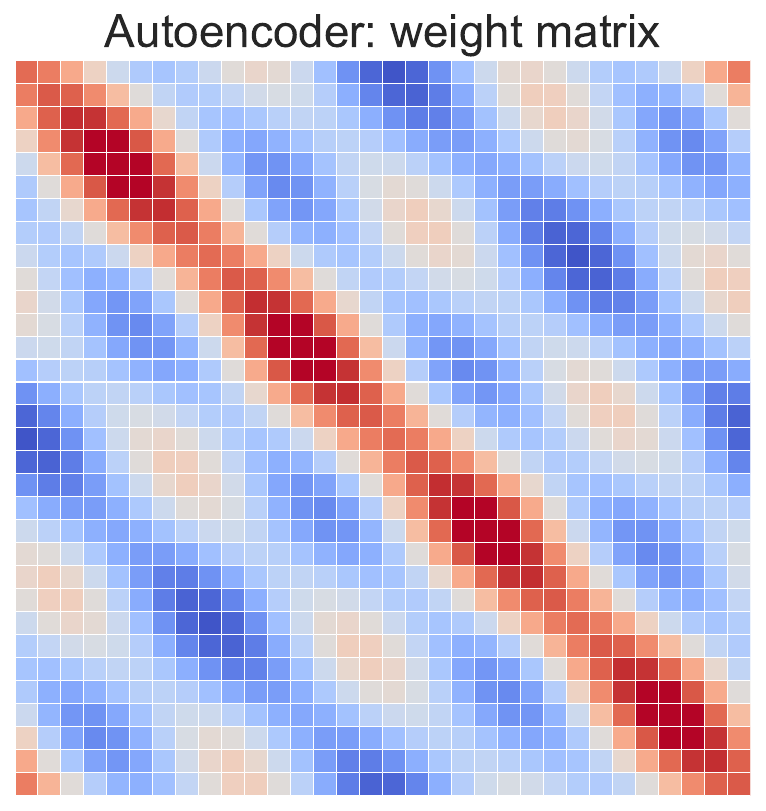}
    \includegraphics[width=0.48\linewidth]{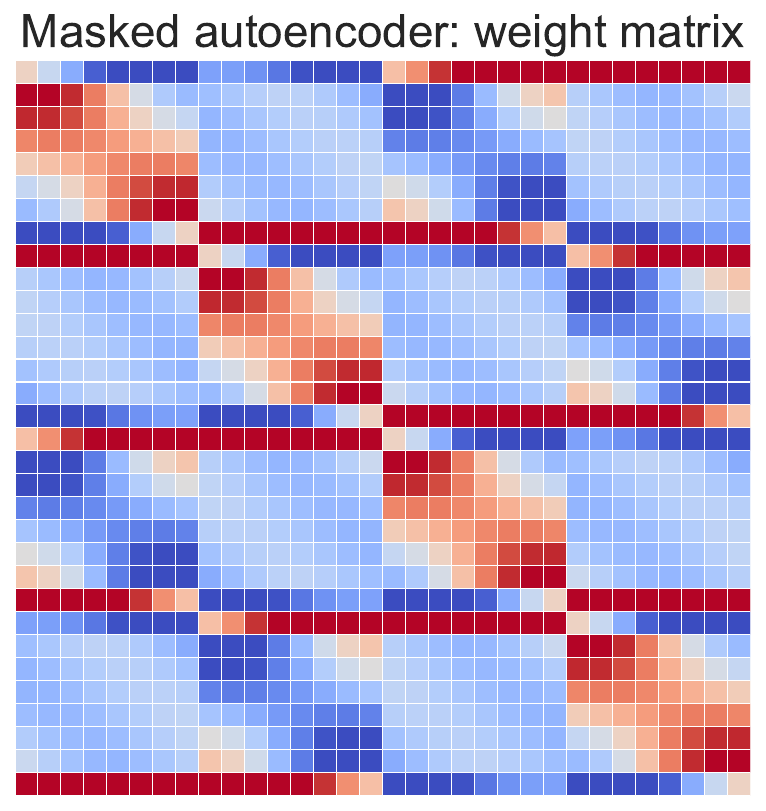}
    \caption{
        We plot the product of the encoder and decoder matrices (bottom) and the encoder weight matrix (top) for an AE (left) and MAE (right). Red indicates weights with higher magnitude, while blue indicates weights with lower magnitude.
        The plots show that MAEs learn an encoder that projects the data onto a basis that is different from a standard autoencoder.
        MAEs prioritizes features at the boundary of the patch since it is most correlated to the other patches and hence most useful for predicting them.
    }
    \label{fig:ising_mae}
\end{figure}
In~\cref{fig:ising_mae}, we plot the encoder weights ($A$) and the product of the encoder and decoder weights ($AB$) for a AE and MAE trained on the Ising model with $d=32$ dimensions and coupling constant $J=2$.
Both the MAE and AE consider an encoder that projects data from 32 to 6 dimensions.
For the MAE, we consider a patch size of 8 and a masking ratio of $m=0.5$. For the AE (\cref{fig:ising_mae} left), the matrix $AB$ has rank 6 (which is the dimensionality of the feature space). The product $AB$ should be close to identity to minimize the autoencoder objective, and this is indeed evident in the experiment.
In~\cref{fig:ising_mae}, the AE encoder learns sinusoidal features of varying frequencies, similar to those learned by linear autoencoders on natural images~\cite{pca_images}.

What is remarkable in this simple experiment is that
an MAE learns a different encoder, one that clearly prioritizes input dimensions at the boundary of the patches. Since data has strong local correlations, the boundary of a patch has the largest correlation to nearby patches and is therefore most useful for reconstructing nearby patches.
\textbf{In summary, MAEs select features that are redundantly present across patches while autoencoders pick features that best explain the variance in the data.}
MAEs are perhaps successful on downstream tasks because many perception tasks are redundant functions of the input~\citep{ramesh2024many}, and MAEs find such redundant features.

\subsection{Features of a linear MAE for natural images}
\begin{figure}[h]
    \centering
    \includegraphics[width=.8\linewidth]{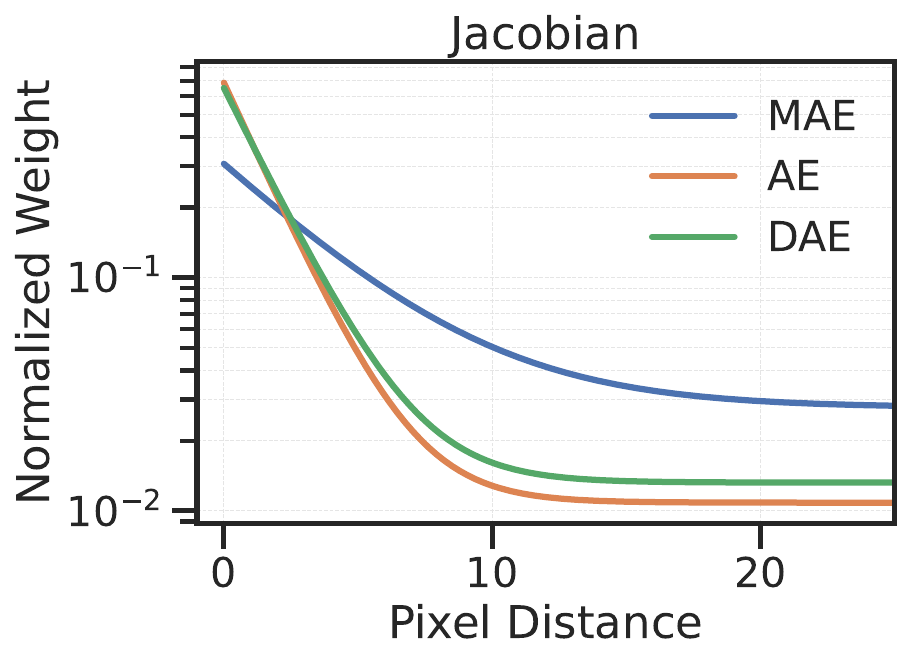}
    \caption{An exponential fit to the magnitude of the entries of the input-output Jacobian as a function of distance from the output pixel from the input pixel of different models (averaged over input and output pixels). Results are shown for three methods: MAEs with 80\%  masking ratio and patch size $p=2$, standard AEs, and Denoising Autoencoders (DAEs) with noise level  $\sigma=0.2$. Experiments are conducted on CIFAR-10 with a latent dimension of 128. In general, MAEs integrate spatial information from distant patches, as opposed to an AE or a DAE.}
    \label{fig:lin_mae_kernel}
\end{figure}

We next train on natural images from CIFAR-10 and inspect how linear MAEs, linear denoising autoencoders (DAEs), and linear AEs encode different types of features. Unlike MAEs, which perform masking, DAEs add Gaussian random noise to the input~\cite{vincent2010stacked}.
For linear encoders and decoders, DAEs (\cref{eq:dae}), denoise additive Gaussian noise with variance $\sigma^2$, and are equivalent to autoencoders with an $\ell_2$ regularization on the weights ($AB$).
\begin{equation}
\label{eq:dae}
\ell_d(A, B) = \mathbb{E}_{L \sim \mathcal{N}(0, \sigma^2)} \left[\left\| X - (X + L) AB \right\|_2^2 \right]
\end{equation}

For AEs, DAEs and MAEs, we compare the average influence that an input pixel has on an output pixel as a function of the distance between the two.
To compute this influence, we consider the Jacobian of the output with respect to the input. For example, for our linear models, the quantity $|(AB)_{ij}|$ determines the influence of input pixel $i$, towards reconstructing target pixel $j$. 
We compute the average magnitude of the normalized absolute weight of the entries of the Jacobian as a function of the distance of the output pixel from the input pixel for linear AEs, DAEs and MAEs trained on CIFAR-10 in ~\cref{fig:lin_mae_kernel}.
We find that AEs learn more localized kernels and DAEs are slightly less localized than AEs.
On the other hand, MAEs integrate spatially distant information.
This is because the MAE objective encourages the use of other patches to reconstruct a patch.

\begin{figure}[h]
    \centering
    \includegraphics[width=\linewidth]{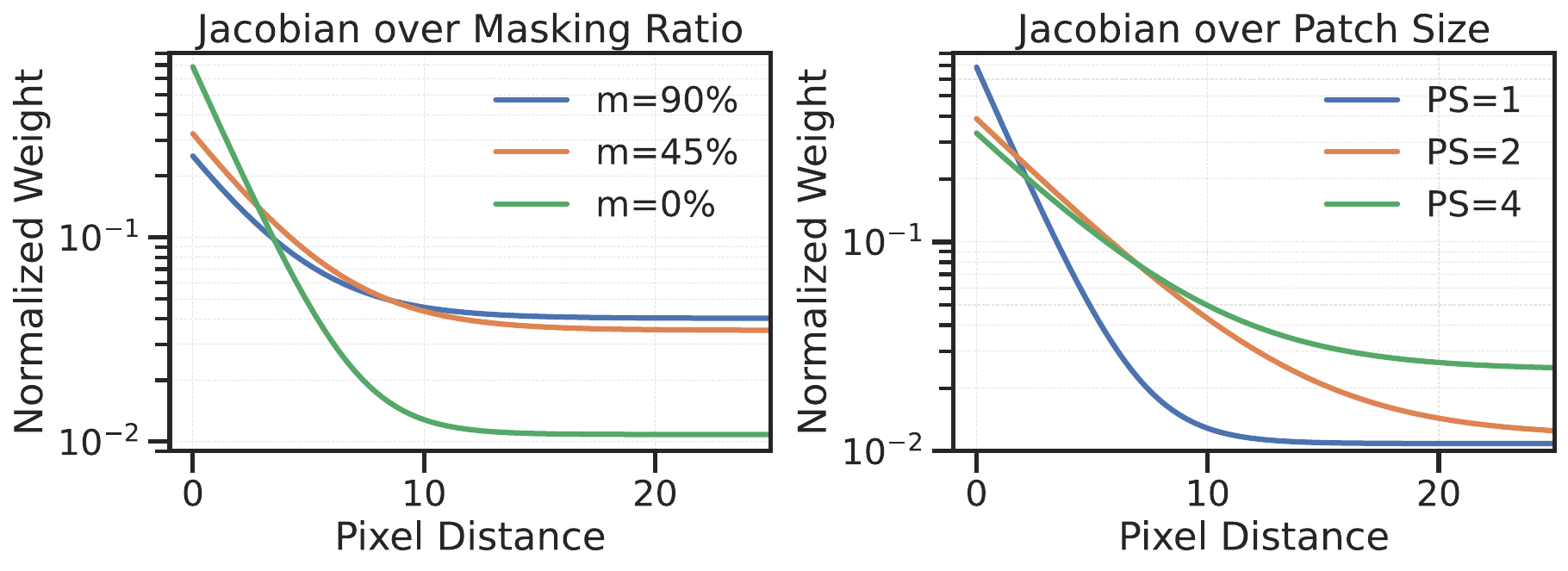}
    \caption{
    An exponential fit to the magnitude of the entries of the input-output Jacobian as a function of distance from the output pixel from the input pixel of different models (averaged over input and output pixels) for different masking ratios and patch sizes. Experiments were conducted for a linear MAE with latent dimension 128 on CIFAR-10. As the masking ratio increases, more and more non-local information is used for reconstruction in a (linear) MAE. Similarly, the reconstruction relies more on nonlocal information as the patch size increases.}
    \label{fig:lin_mae_masking}
\end{figure}

Next, we investigate how MAE hyperparameters---masking ratio $m$ and patch size $p$ --- affect the types of learned features. Mathematically, we can see this influence from \cref{eq:A_opt}, where the term $V^{-1}X^\top X$ first performs a whitening transformation of the original data. As the masking ratio increases, this transformation effectively performs patch-wise whitening, making all intra-patch correlations identity and normalizing inter-patch correlations. Masking forces the model to incorporate non-local information beyond the immediate patch for reconstruction, as seen in ~\cref{fig:lin_mae_masking}. Our results reveal that a high masking ratio produces a more diffuse average Jacobian, whereas a low masking ratio yields a more localized one. We observe a similar relationship with patch size, aligning with our prediction in \cref{eq:A_opt}: larger patch sizes expand the integration of information from regions outside the patch. In both of these cases, increasing the masking ratio and patch size reduces the weight of intra-patch pixels.  

\subsection{What kinds of tasks benefit from MAEs compared to AEs?}
\begin{figure}[h]
    \centering
    \includegraphics[width=\linewidth]{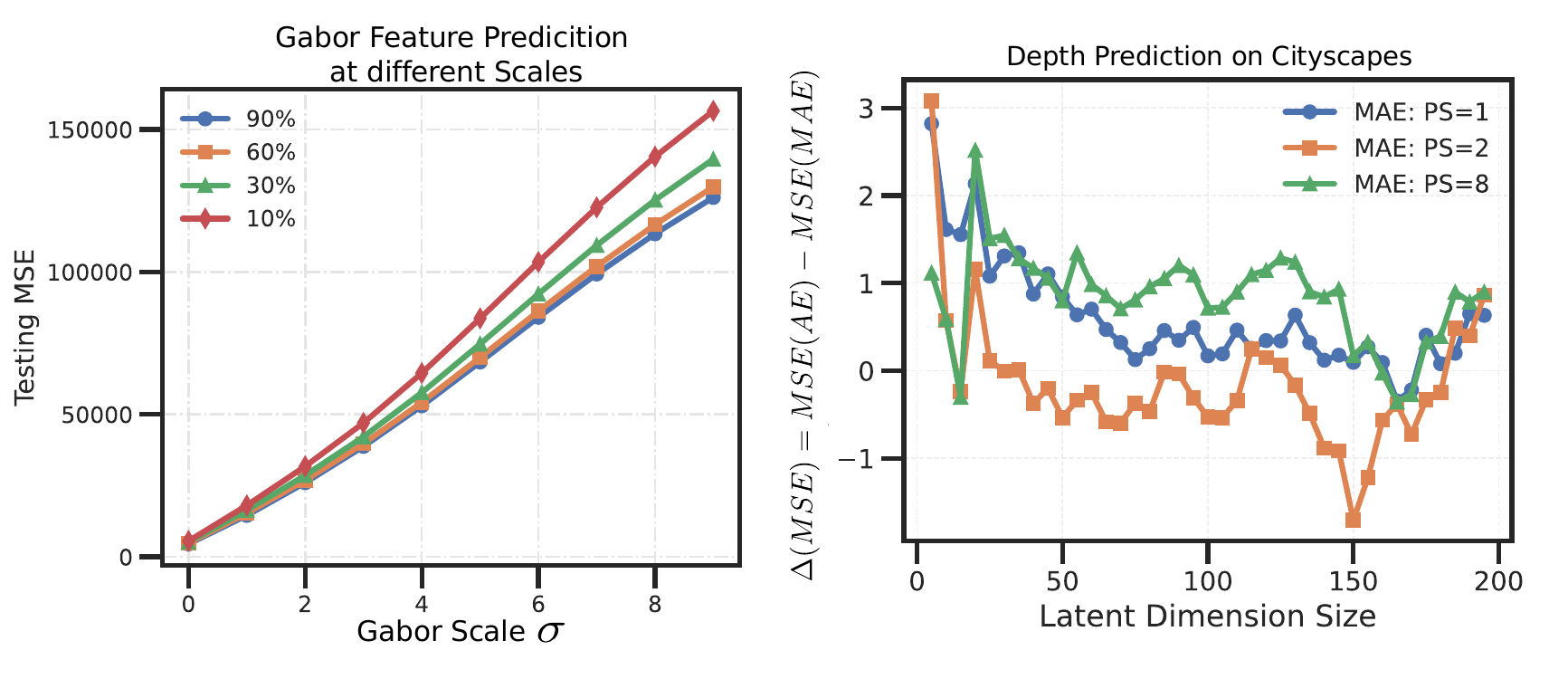}
    \caption{\textbf{Left:} Gabor feature prediction task using pretrained features from MAEs as a function of spatial scale ($\sigma$). In general, we see that a higher masking ratio, which forces the MAE to learn more non-local information, performs better on larger scales. \textbf{Right:} Reduction in MSE when an MAE is used for Cityscapes depth prediction as compared to an AE. For low-dimensional latents, MAEs outperform AEs on supervised depth prediction.}
    \label{fig:lin-gabor-task}
\end{figure}

Previously, we showed that MAEs learn different features than AEs and DAEs, as illustrated in \cref{fig:lin_mae_kernel}, e.g.,
MAEs capture more spatially distant information while reducing the emphasis on intra-patch details.
This raises a key question: what kinds of tasks benefit more from MAEs compared to AEs?
We investigate this through two tasks: a synthetic Gabor feature prediction task requiring integration of spatially distant information, and a monocular depth prediction task using the Cityscapes~\cite{Cordts2016TheCD} dataset.

\cref{fig:lin_mae_masking} suggests that MAEs capture more spatially distant information as the masking ratio and patch size increase. To investigate this, we evaluate them on predicting Gabor features, a task that allows precise control over the spatial dependencies required for accurate prediction. Gabor functions, denoted as $g(i,j)$ for spatial coordinates $i,j$, are localized harmonics modulated by a Gaussian window and are parameterized by frequency $f$, phase shift $\phi$, Gaussian scale $\sigma$, and dilation $\gamma$:
\begin{equation}
g(i,j) = \exp\left( -\frac{i^2 + \gamma^2 j^2}{2\sigma^2} \right) \cos\left( 2\pi i f + \phi \right).
\end{equation}
Our target, $x_g(i, j)$, is the convolution of the Gabor function over the entire image
\(
    x_g(i,j) = (x * g)(i,j).
\)
The key advantage of using a Gabor feature prediction task is that we can adjust the parameter  $\sigma$ to control whether the task depends on local or long-range spatial interactions. This design enables a systematic evaluation of whether MAEs trained with higher masking ratios are better at capturing long-range dependencies. To analyze this effect, we trained an MAE with patch size ($p=2$) and frequency values $f=[0.03, 0.06, 0.1]$, with fixed $\phi=0$ and $\gamma=1$, across various masking ratios (\cref{fig:lin-gabor-task}).  For small $\sigma$ values, different methods perform similarly. However, as $\sigma$ increases, MAEs trained with high masking ratios outperform those with lower masking ratios. This aligns with our findings in \cref{fig:lin_mae_kernel}, where MAEs integrate information from distant spatial patches.  

While the Gabor prediction task provides a controlled setting to analyze MAEs, real-world perception tasks present greater complexity. To assess how MAEs perform in practical scenarios, we evaluate them on monocular depth prediction---a task that requires integrating spatial information to infer depth accurately. We trained MAEs on the Cityscapes dataset, downsampled to 32×32 resolution, and analyzed performance as a function of latent space dimensionality. For small latent dimensionalities, all patch sizes ($m=0.8$) outperform an AE. This finding highlights how MAEs, through masking-based regularization, are useful for biasing representations for downstream perception tasks. 

Linear MAEs are amenable to analysis but they come with their own limitations.
\begin{enumerate*}[(i)]
    \item On datasets like CIFAR, ViTs use over-complete representations of the data, while our analysis is restricted to rank-deficient or full-rank encoders and decoders.
    \item  Linear MAEs only capture 2nd-order correlations of the data. Higher order correlations in the data distribution can be used for better reconstruction, and possibly provide a better prior for downstream tasks.
    \item The network architecture for linear MAEs is limited to fully connected networks, which provides limited insight into designing ViTs for MAEs.
\end{enumerate*}

\begin{summary}
  \begin{enumerate}[noitemsep,topsep=0pt,leftmargin=*]
    \item Linear MAEs learn a reconstruction kernel that biases feature learning towards spatially distant and correlated parts of the input by discouraging the use of intra-patch level information through the regularizer $||GAB||^2$.
    \item  Increasing the masking ratio or patch size forces MAE features to retain spatially distant information. 
    \item MAEs with high masking ratios perform better at tasks that require incorporating inter-patch level information. 
  \end{enumerate}
\end{summary}

\section{Understanding Nonlinear MAEs}
The objective for a non-linear MAE can be written as
\begin{equation}\label{eq:nolinmae_objective}
    \ell_m(\theta) = \mathbb{E}_{R} \left[\left\| X - f_\theta(g_\theta(R \odot X)) \right\|^2\right],
\end{equation}
where $g_{\theta}$ is the encoder and $f_{\theta}$ is the decoder, and we define: $f_\theta(g_\theta(R\odot X))=h_{\theta}(R\odot X) $.
For MAEs, the encoder and decoder are ViT blocks, where the encoder processes only the unmasked patches, and the decoder processes all patches.
We could extend our analysis to non-linear MAEs using a linear approximation of MAEs (see ~\cref{app:non_linear_mae}). However, these approximations give us little insight into how we should train MAEs using deep networks.
  
To analyze the encoder and decoder, we adopt the technique of \citet{Kadkhodaie2023GeneralizationID}, which draws a connection between diffusion models and denoising. We can draw a similar parallel between MAEs and denoisers, except that the denoising task is set up using masking instead of corruption by additive Gaussian noise. \citet{Kadkhodaie2023GeneralizationID} argue that diffusion models learn an adaptive low-dimensional basis from data and we can identify this basis using techniques from~\citet{Mohan2019RobustAI}.
Extracting this basis requires the network to be a first-order homogeneous function, which ViTs are not---operations like layer-norm and attention are not homogeneous. To understand the basis learned by MAEs trained using ViTs, we approximate our denoiser $h_\theta(R\odot x)$ via a first-order Taylor series expansion around an input $\tilde{x} = R \odot x$, denoted by $x_0$. The approximation can be written as:
\begin{equation}
\label{eq:taylor_jacob}
h(\tilde{x}) \approx h(x_0) + \nabla_{\tilde{x}} h(\tilde{x})(\tilde{x} - x_0) = A (\tilde{x}) + b,
\end{equation}
where $ A = \nabla_{\tilde{x}} h(\tilde{x}) $ is the Jacobian of the reconstruction with respect to the inputs and $ b = f(x_0) - \nabla_{\tilde{x}} h(\tilde{x}) x_0 $ is the bias term.
Assuming that the Jacobian does not change locally~\citep{Balestriero2018AST}, $\nabla_{\tilde{x}}h(\tilde{x})=\nabla_{x_0}h(x_0)$. We use this technique as a heuristic to analyze a nonlinear MAE, and thus interpret the encoder and decoder as learning an adaptive basis matrix $A = \nabla_{\tilde{x}} h(\tilde{x})$ for each data sample.  

\begin{figure}
    \centering
    \includegraphics[width=\linewidth]{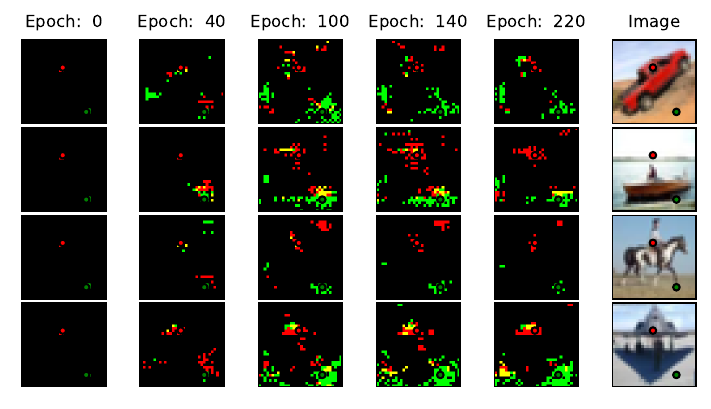}
    \caption{Visualization of the Jacobian of the output with respect to CIFAR-10 images, showing which input pixels contribute to reconstructing a given output pixel. We track this throughout training for the green and red pixel locations, showing their respective reconstruction kernels in corresponding colors. Each row presents a different image. In general, the Jacobian evolves throughout training, transitioning from highly localized to broader spatial dependencies, and adapts dynamically to each sample. The network is trained with a 80\% masking ratio and a patch size of 2. }
    \label{fig:vit-recon-kernel}
    \vspace{-3mm}
\end{figure}

Fig.~\ref{fig:vit-recon-kernel} illustrates the evolution of the learned basis throughout training for a MAE-trained ViT with a masking ratio of 80\% and a patch size of 2. The figure highlights two different output pixels, circled in red and green, with corresponding colored input features identified by the Jacobian. First, unlike the linear MAEs discussed earlier, the basis used for reconstruction is adaptive and varies across images rather than being fixed. Second, during training, the basis transitions from being highly localized at the start to more diffuse, incorporating information from distant regions as training progresses. Third, while the linear model relies primarily on second-order correlations, ViTs leverage higher-order correlations to improve the performance on masked reconstruction. For instance, in the horse image, the reconstruction kernel adapts to capture correlations with the human’s shirt, while for background pixels, it dynamically adjusts across similar regions. This higher-order correlation modeling mitigates the overly smoothed reconstructions typical of purely second-order models, demonstrating that our model leverages complex statistical dependencies.

\begin{figure}
    \centering
    \includegraphics[width=\linewidth]{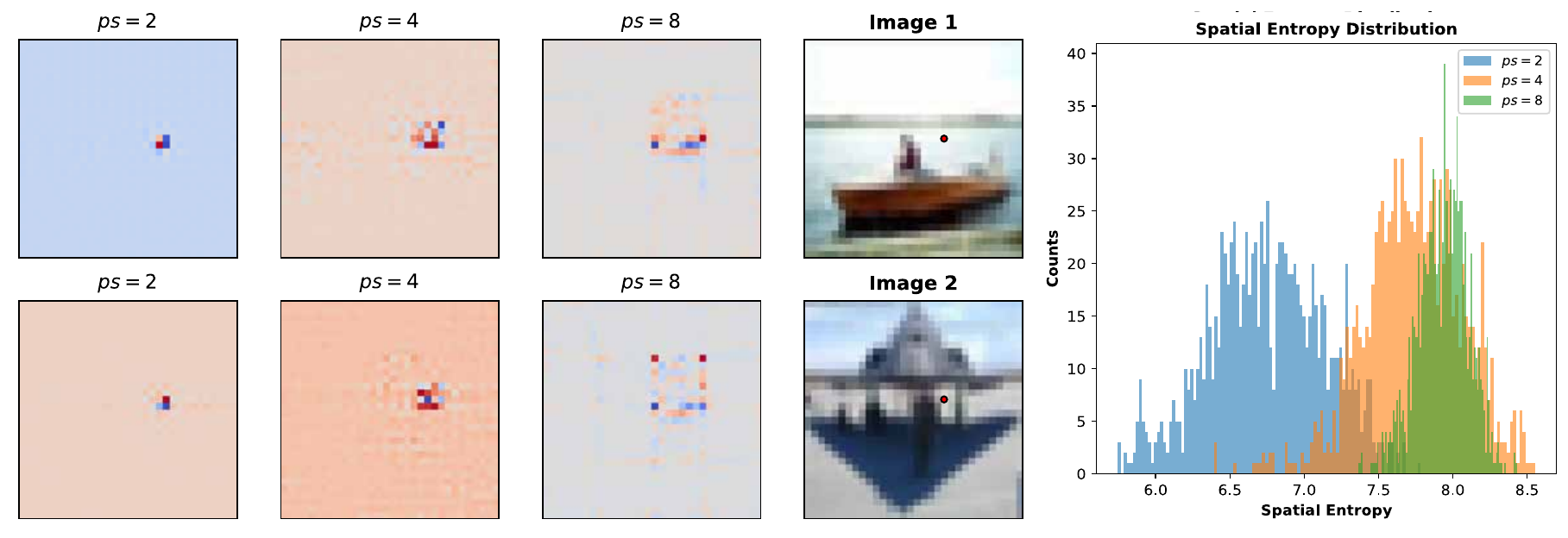}
    \caption{\textbf{Left:} Visualization of the Jacobian of the output with respect to CIFAR-10 images as we vary patch size (masking ratio fixed  at $m=0.6$) for the image in each row, for the denoted output pixel in red. As we increase the patch size, the weights have a less localized distribution. \textbf{Right:} Spatial Entropy distribution over all the output pixels for different patch sizes. In general, larger patch sizes result in distributions with higher spatial entropy. }
    \label{fig:no-lin-ps}
    \vspace{-5mm}
\end{figure}

Next, we examine the influence of masking ratio and patch size on the Jacobian matrix $A$. In \cref{fig:no-lin-ps}, we demonstrate that increasing the patch size shifts the reconstruction from relying on local information to incorporating more global context. We characterize this using the spatial entropy for each output pixel. Given that the Jacobian of a single output pixel is denoted as $ J \in \mathbb{R}^{HW} $, the spatial entropy is computed as:  
\begin{equation}
    S = -\sum_i \frac{|J_i|}{\|J\|_1} \log \left( \frac{|J_i|}{\|J\|_1} \right),
\end{equation}
where $i$ is the ith component of $J$.
 We then construct a histogram of $ S $ for all output pixels.
 Generally, larger patch sizes result in higher spatial entropy, indicating broader use of information compared to smaller patches, while changes in masking ratio had a much weaker influence on distant spatial information. Beyond CIFAR, we also show similar results on ImageNet-64 in \cref{fig:imgnet-jacob}.

\begin{summary}
  \begin{enumerate}
    \item Nonlinear MAEs develop an adaptive kernel for denoising that is specific to the test datum compared to the dataset-specific kernel of a linear MAE.
    \item Non-linear MAEs grow over the course of training from a well-localized kernel at the start to a kernel with a broader spatial support at the end of training.
    \item Similar to the linear MAE, increasing patch size increases the spatial extent of encoded features. 
\end{enumerate}
\end{summary}
\section{From theory to practice: How to train your MAE}
Masked autoencoders require significantly more training epochs than other self-supervised learning methods to achieve competitive downstream performance. For example, on ImageNet, MAEs require 1,600 training epochs, whereas I-JEPA~\cite{Assran2023SelfSupervisedLF} achieves a similar task performance in just 300 epochs. Both pretraining and fine-tuning can be sensitive to the choice of hyperparameters, particularly with ViTs~\citep{steiner2021train}.
In this section, we discuss some empirical insights for MAE pretraining and fine-tuning based on our experiments on CIFAR-10.

\begin{figure}[!t]
    \centering
    \includegraphics[width=0.45\linewidth]{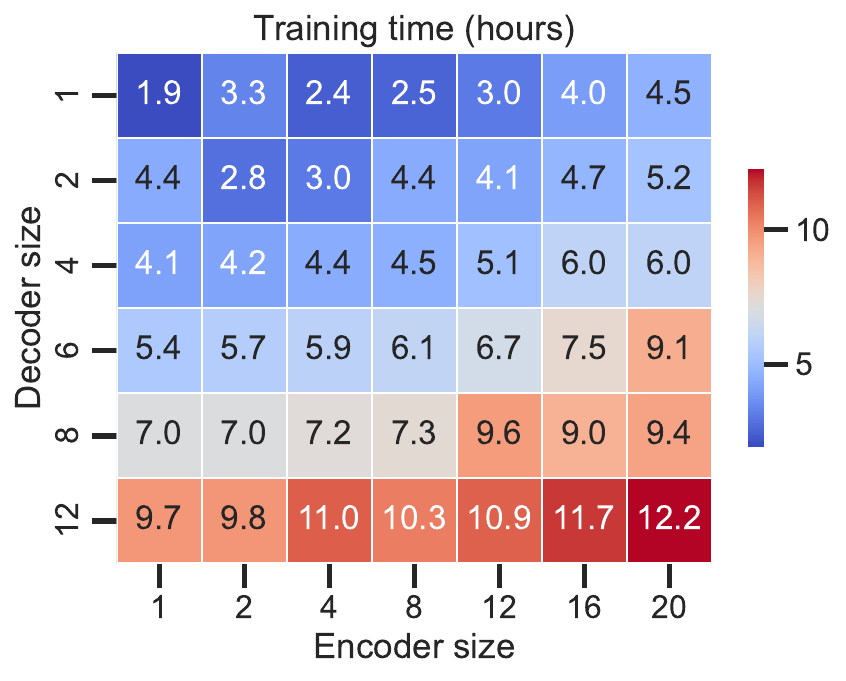}
    \includegraphics[width=0.45\linewidth]{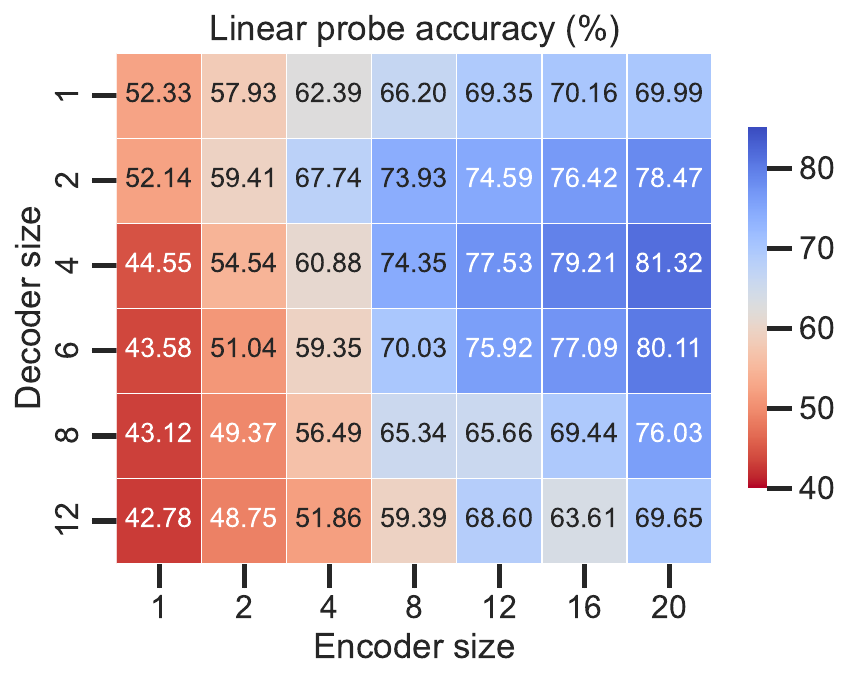}
    \caption{
        We plot (left) the training time and (right) linear probe accuracy for different combinations of the number of encoder and decoder layers.
        \textbf{MAEs are slow to train} with training time growing faster as we increase the size of the decoder.
        We find that training loss is not a good proxy for downstream classification accuracy~\cref{fig:decoder_size_loss_time}.
        The \textbf{accuracy of the trained encoder continues to improve as we increase its size}. However, this is not true for the decoder: the optimal decoder contains 2-4 layers. We notice that the differences in accuracies reduce when we fine-tune the encoder for a 100 epochs~\cref{fig:decoder_size_acc}
        }
    \label{fig:decoder_time}
    \vspace{-3mm}
\end{figure}
\subsection{Experimental details}
We train MAEs using the architecture in~\citet{He2021MAE}.
We fix the embedding dimension $d$ at 192, the number of attention heads in the encoder to 12 and the decoder to 16 in all our experiments.
The Transformer blocks use the GeLU non-linearity~\citep{hendrycks2016gaussian} with the layer-norm before the attention and MLP blocks ---  also known as the pre layer-normalization transformer~\citep{xiong2020layer}.

The MAE is trained for 2000 epochs using the AdamW optimizer with 50 epochs of warmup and a cosine-annealed learning rate schedule with an initial learning rate of $1.2 \times 10^{-3}$.
The MAEs are trained with a batch-size of 4096 and weight-decay of 0.05.
After pretraining, we discard the decoder and fine-tune the encoder.
We fine-tune for only 100 epochs using a batch-size of 1024 with 5 epochs of warmup and an initial learning rate of $10^{-3}$ annealed to $10^{-5}$ using a cosine-annealed schedule.
For MAE pretraining, we use random-resize crop and horizontal flips as the two augmentations, while for fine-tuning we use RandAugment~\citep{cubuk2020randaugment}.

We evaluate the accuracy on the downstream task using a linear-probe applied to the encoder. We find that trends obtained using a linear probe are nearly identical to trends obtained using fine-tuning (see~\cref{app:expts}).

\begin{figure}[!t]
    \centering
    \includegraphics[width=0.49\linewidth]{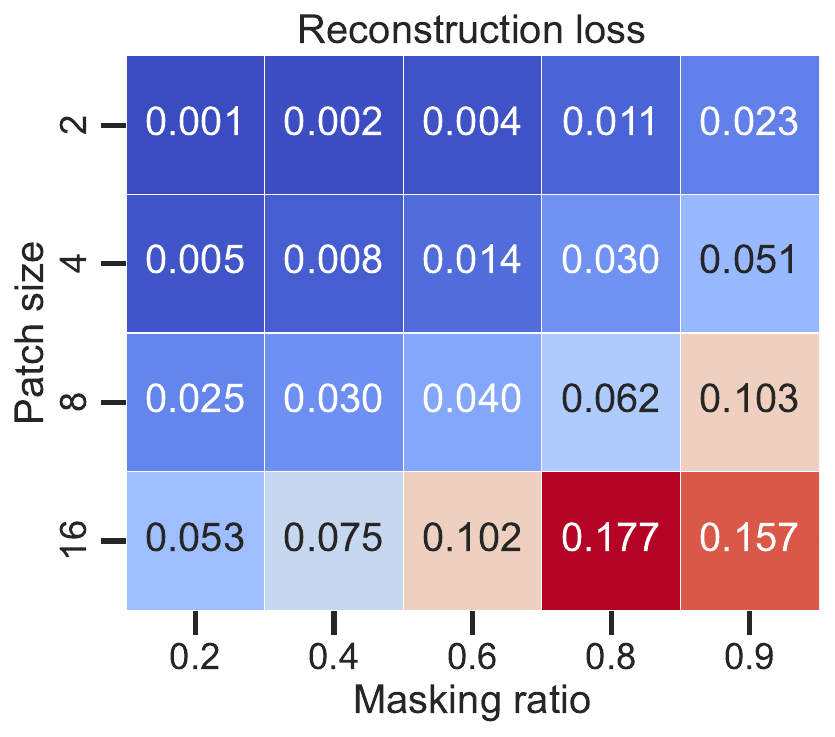}
    \includegraphics[width=0.49\linewidth]{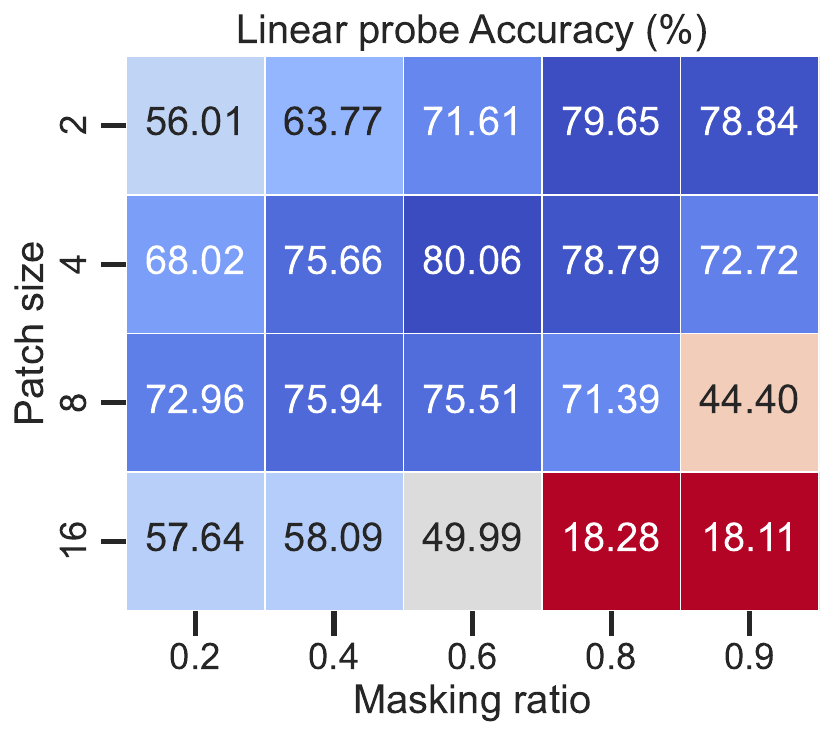}
    \caption{
        We plot the reconstruction error (left) and the linear probe accuracy (right) for different values of patch-size and masking ratio.
        Reconstruction loss decreases as we decrease the patch-size or decrease the masking ratio. However the linear probe accuracy has no clear trend with respect to the masking ratio and the optimal masking ratio depends on the patch-size. While smaller values of patch-size achieve better linear probe accuracies, they are also slower to train.
    }
    \label{fig:masking_patch}
    \vspace{-7mm}
\end{figure}
\subsection{Results} 
\textbf{How should we select the size of the encoder and decoder?}
Masked autoencoders typically have more encoder layers than decoder layers, but is this optimal? We investigate how the number of encoder and decoder layers affect the reconstruction error, training time and downstream classification accuracy.
In~\cref{fig:decoder_time}, we train models on CIFAR-10 and plot the training time (in hours) as we vary the number of encoder and decoder layers.
The training time increases at a slower rate when we increase the number of encoder layers compared to increasing the number of decoder layers since the decoder operates on a longer sequence of tokens --- the decoder uses both masked and unmasked tokens while the encoder only uses unmasked tokens. 
Even for CIFAR-10, training can take as long as 6 to 10 hours on a single GPU.

Increasing the number of encoder and decoder layers decreases the reconstruction error (\cref{fig:decoder_size_loss_time}). 
However, a smaller reconstruction error does not imply higher linear-probe accuracies (\cref{fig:decoder_size_loss_time}) or fine-tuning accuracies (\cref{fig:decoder_size_acc}).
\textbf{We find that the linear-probe accuracy continues to improve as we increase the size of the encoder.
However, the optimal decoder has far fewer layers and even a 1-layer decoder has high accuracy after fine-tuning with the added benefit of reduced training time.}
For example, a model pretrained with a 1-layer decoder achieves 95.26\% accuracy after fine-tuning, which is only 0.30\% worse than our best model~\cref{fig:decoder_size_acc}.
\begin{figure}[h]
    \centering
    \includegraphics[width=0.47\linewidth]{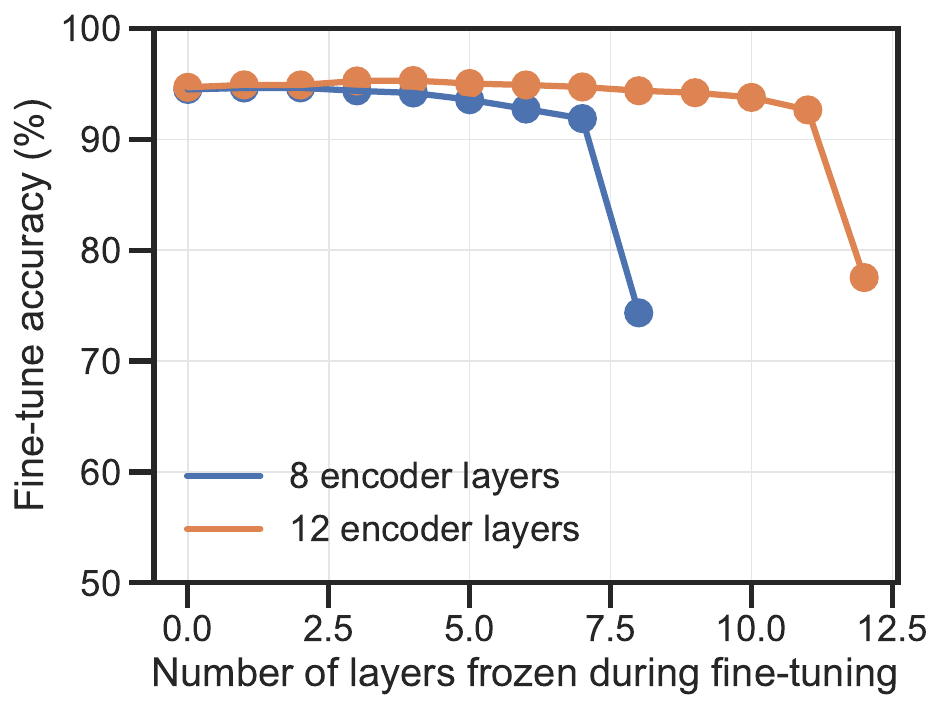}
    \includegraphics[width=0.47\linewidth]{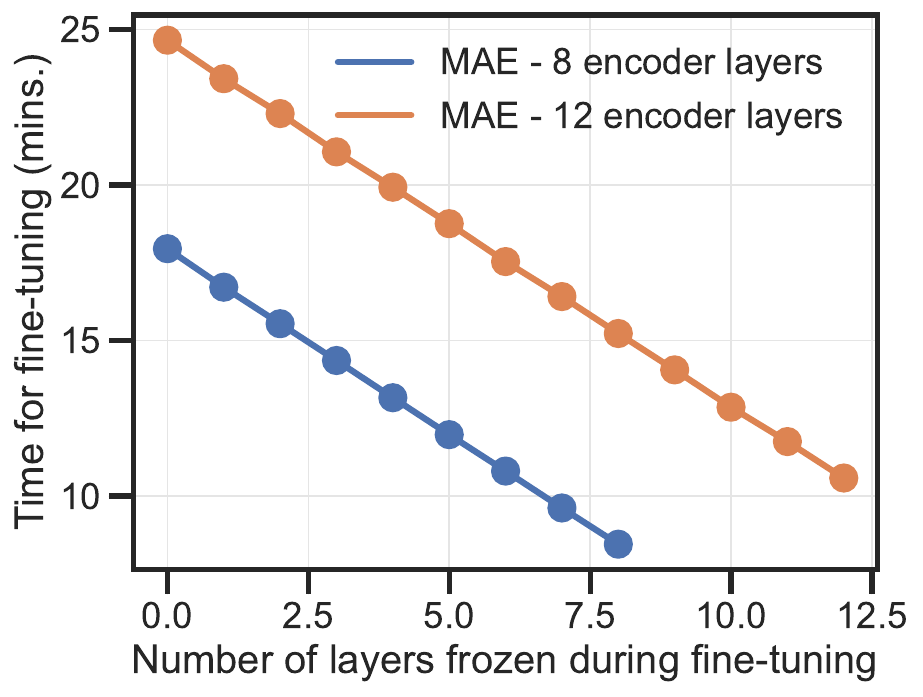}
    \caption{
        We fine-tune the model for 100 epochs after freezing the first $k$ layers of the network and plot the accuracy against number of layers frozen during fine-tuning. We find that even if we freeze all but 1 Transformer block, we recover the accuracy compared to when no layers are frozen, i.e., the accuracy can be recovered while using a fraction of the training time and memory.
    }
    \label{fig:decoder_blocks}
\end{figure}

\textbf{How do we select the masking ratio and patch-size?}
The masking ratio and patch-size control the basis learned by a linear MAE (\cref{s:linear_mae}). 
We see that these two hyper-parameters also have a significant impact on MAEs trained using deep networks.
In~\cref{fig:masking_patch}, we plot the reconstruction error and the linear-probe accuracy for different values of masking ratio and patch-size.
\textbf{MAEs train faster if the patch-size is large or if the masking ratio is large} since a larger patch-size reduces the number of tokens by a quadratic factor and increasing the masking ratio decreases the number of tokens fed to the encoder. The training time increases from 4 hours to 10 hours as we decrease the masking ratio from 0.9 to 0.1.

Reconstruction improves as we decrease the masking ratio which is unsurprising --- it is easier to reconstruct an image if we have more patches.
Reconstruction error also reduces as the patch-size becomes smaller.
A smaller patch-size reduces the average spatial distance to unmasked patches, which makes reconstruction easier due to strong local spatial correlations seen in images (\cref{fig:local_correlations}).
We find that a small patch-size and large masking ratio achieve the best linear probe accuracies. While smaller values of patch-size achieve better linear probe accuracies, an MAE with larger patch-size trains faster.

\textbf{Can we fine-tune fewer layers to reduce GPU memory and training time?}
We investigate the effect of fine-tuning a subset of the layers of an MAE in~\cref{fig:decoder_blocks}.
We find that \textbf{if we freeze all but one Transformer block, we achieve within 2\% of the accuracy achieved from fine-tuning all the layers}, i.e.,  a similar accuracy can be achieved with a nearly 2x speedup (\cref{fig:decoder_size_acc} (right)) and with less GPU memory. These results align with the practice of setting learning rates that exponentially decrease across the different layers of the encoder during fine-tuning~\citep{He2021MAE}.

\begin{summary}
  \begin{enumerate}
    \item MAEs achieve higher linear probe accuracy with more encoder layers and fewer decoder layers. On CIFAR-10, accuracy improves consistently with more layers, and a single-layer decoder reaches within 1\% of the best fine-tuning accuracy while training 4× faster. 
    \item As the patch-size used to train the MAE increases, the masking ratio that achieves the highest linear-probe accuracy decreases.
    \item On CIFAR-10, fine-tuning just the last transformer block is enough to achieve within 2\% of the best accuracy. 
  \end{enumerate}
\end{summary}


\section{Related Work}

Theoretical work on understanding MAEs can be categorized into the following themes: architectural modifications, data modeling, connections to other self-supervised learning objectives and denoising. 
Previous works that study architectural variants find that masking alters the attention structure of ViTs, creating more local attention in early layers rather than global~\cite{Cao2022HowTU, Shekhar2022UnderstandingCV, xie_revealing_2023, park_what_2023, huang_how_2024}. \citet{Raghu2021DoVT} suggests that ViTs with local receptive fields generalize better to vision tasks. \citet{Kong2023UnderstandingMA} hypothesizes that this learning process attempts to learn a hierarchical latent variable model, where masking enforces identifiability of coarser latent variables, though such a model is not definitively known to exist for natural image data.
Another perspective interprets MAEs as a form of contrastive learning using only positive samples of unmasked patches~\cite{Zhang2022HowMM}. \citet{littwin2024jepa} challenges this hypothesis by demonstrating that MAE approaches learn a different set of features than contrastive methods.
Beyond explaining the architectural bias induced by MAEs, researchers have also investigated how MAEs can perform masked image modeling tasks under extreme masking ratios. This challenge can be framed as a linear inverse problem, where measurements of unmasked pixels are used to recover masked pixels~\cite{Kadkhodaie2021StochasticSF}. While compressive sensing offers one method to solve this inverse problem, it lacks the nonlinear encoder and prior over the data distribution present in MAEs~\cite{Wright2022HighDimensionalDA}. For the nonlinear problem, the field of image denoising provides a probabilistic understanding of the unique properties induced in this denoising task~\cite{BioucasDias2009MultiplicativeNR, Milanfar2024DenoisingAP}. For additional related work see~\cref{s:app_related}
\section{Conclusion}

This work sheds light on how MAEs build strong foundational features. We showed that MAEs capture spatial correlations in input images. We model MAEs in a linear setting to demonstrate how masking ratio and patch size affect the integration of spatially distant features. We extend this study to nonlinear MAEs by analyzing the input-output Jacobian and provide practical training guidelines for practitioners.

While MAEs are adaptable, the sensitivity of their hyperparameters to dataset and task characteristics necessitates careful tuning. Understanding the connection between properties of the visual task and optimal MAE hyperparameters remains an important direction for future work. Overall, our characterization of MAEs as spatial correlation learners provides valuable insights for advancing their theoretical understanding.
{
    \small
    \bibliographystyle{ieeenat_fullname}
    \bibliography{main}

\begin{thebibliography}{57}
\providecommand{\natexlab}[1]{#1}
\providecommand{\url}[1]{\texttt{#1}}
\expandafter\ifx\csname urlstyle\endcsname\relax
  \providecommand{\doi}[1]{doi: #1}\else
  \providecommand{\doi}{doi: \begingroup \urlstyle{rm}\Url}\fi

\bibitem[Assran et~al.(2023)Assran, Duval, Misra, Bojanowski, Vincent, Rabbat, LeCun, and Ballas]{Assran2023SelfSupervisedLF}
Mahmoud Assran, Quentin Duval, Ishan Misra, Piotr Bojanowski, Pascal Vincent, Michael~G. Rabbat, Yann LeCun, and Nicolas Ballas.
\newblock Self-supervised learning from images with a joint-embedding predictive architecture.
\newblock \emph{CVPR}, pages 15619--15629, 2023.

\bibitem[Baevski et~al.(2022)Baevski, Hsu, Xu, Babu, Gu, and Auli]{baevski2022data2vec}
Alexei Baevski, Wei-Ning Hsu, Qiantong Xu, Arun Babu, Jiatao Gu, and Michael Auli.
\newblock Data2vec: A general framework for self-supervised learning in speech, vision and language.
\newblock In \emph{ICML}, pages 1298--1312, 2022.

\bibitem[Bai et~al.(2023)Bai, Wang, Xiao, Wei, Wang, Yuille, Zhou, and Xie]{Bai2022MaskedAE}
Yutong Bai, Zeyu Wang, Junfei Xiao, Chen Wei, Huiyu Wang, Alan~Loddon Yuille, Yuyin Zhou, and Cihang Xie.
\newblock Masked autoencoders enable efficient knowledge distillers.
\newblock \emph{CVPR}, pages 24256--24265, 2023.

\bibitem[Baldi and Hornik(1989)]{baldi1989neural}
Pierre Baldi and Kurt Hornik.
\newblock Neural networks and principal component analysis: Learning from examples without local minima.
\newblock \emph{Neural networks}, 2\penalty0 (1):\penalty0 53--58, 1989.

\bibitem[Balestriero and Baraniuk(2018)]{Balestriero2018AST}
Randall Balestriero and Richard Baraniuk.
\newblock A spline theory of deep learning.
\newblock In \emph{ICML}, 2018.

\bibitem[Balestriero and LeCun(2024)]{balestriero2024learning}
Randall Balestriero and Yann LeCun.
\newblock Learning by reconstruction produces uninformative features for perception.
\newblock In \emph{ICML}, 2024.

\bibitem[Bioucas-Dias and Figueiredo(2009)]{BioucasDias2009MultiplicativeNR}
Jos{\'e}~M. Bioucas-Dias and M{\'a}rio A.~T. Figueiredo.
\newblock Multiplicative noise removal using variable splitting and constrained optimization.
\newblock \emph{IEEE Transactions on Image Processing}, 19:\penalty0 1720--1730, 2009.

\bibitem[Cao et~al.(2022)Cao, Xu, and Clifton]{Cao2022HowTU}
Shuhao Cao, Peng Xu, and David~A. Clifton.
\newblock How to understand masked autoencoders.
\newblock \emph{ArXiv}, abs/2202.03670, 2022.

\bibitem[Chen et~al.(2015)Chen, Weinberger, Xu, and Sha]{Chen2015MarginalizingSL}
Minmin Chen, Kilian~Q. Weinberger, Zhixiang~Eddie Xu, and Fei Sha.
\newblock Marginalizing stacked linear denoising autoencoders.
\newblock \emph{J. Mach. Learn. Res.}, 16:\penalty0 3849--3875, 2015.

\bibitem[Cordts et~al.(2016)Cordts, Omran, Ramos, Rehfeld, Enzweiler, Benenson, Franke, Roth, and Schiele]{Cordts2016TheCD}
Marius Cordts, Mohamed Omran, Sebastian Ramos, Timo Rehfeld, Markus Enzweiler, Rodrigo Benenson, Uwe Franke, Stefan Roth, and Bernt Schiele.
\newblock The cityscapes dataset for semantic urban scene understanding.
\newblock \emph{CVPR}, pages 3213--3223, 2016.

\bibitem[Cubuk et~al.(2020)Cubuk, Zoph, Shlens, and Le]{cubuk2020randaugment}
Ekin~D Cubuk, Barret Zoph, Jonathon Shlens, and Quoc~V Le.
\newblock Randaugment: Practical automated data augmentation with a reduced search space.
\newblock In \emph{NeurIPS}, 2020.

\bibitem[Devlin et~al.(2019)Devlin, Chang, Lee, and Toutanova]{Devlin2019BERTPO}
Jacob Devlin, Ming-Wei Chang, Kenton Lee, and Kristina Toutanova.
\newblock Bert: Pre-training of deep bidirectional transformers for language understanding.
\newblock In \emph{North American Chapter of the Association for Computational Linguistics}, 2019.

\bibitem[Dosovitskiy et~al.(2021)Dosovitskiy, Beyer, Kolesnikov, Weissenborn, Zhai, Unterthiner, Dehghani, Minderer, Heigold, Gelly, Uszkoreit, and Houlsby]{dosovitskiy2020image}
Alexey Dosovitskiy, Lucas Beyer, Alexander Kolesnikov, Dirk Weissenborn, Xiaohua Zhai, Thomas Unterthiner, Mostafa Dehghani, Matthias Minderer, Georg Heigold, Sylvain Gelly, Jakob Uszkoreit, and Neil Houlsby.
\newblock An image is worth 16x16 words: Transformers for image recognition at scale.
\newblock \emph{ICLR}, 2021.

\bibitem[d’Ascoli et~al.(2021)d’Ascoli, Touvron, Leavitt, Morcos, Biroli, and Sagun]{d2021convit}
St{\'e}phane d’Ascoli, Hugo Touvron, Matthew~L Leavitt, Ari~S Morcos, Giulio Biroli, and Levent Sagun.
\newblock {ConViT}: Improving vision transformers with soft convolutional inductive biases.
\newblock In \emph{ICML}, 2021.

\bibitem[Fan et~al.(2023)Fan, Wang, Liao, Zhu, Bhat, Santos-Villalobos, Rohith, and Li]{Fan2023MotionGuidedMF}
David Fan, Jue Wang, Shuai Liao, Yi Zhu, Vimal Bhat, Hector~J. Santos-Villalobos, M.~V. Rohith, and Xinyu Li.
\newblock Motion-guided masking for spatiotemporal representation learning.
\newblock \emph{ICCV}, 2023.

\bibitem[Fang et~al.(2023)Fang, Wang, Xie, Sun, Wu, Wang, Huang, Wang, and Cao]{fang2023eva}
Yuxin Fang, Wen Wang, Binhui Xie, Quan Sun, Ledell Wu, Xinggang Wang, Tiejun Huang, Xinlong Wang, and Yue Cao.
\newblock Eva: Exploring the limits of masked visual representation learning at scale.
\newblock In \emph{CVPR}, 2023.

\bibitem[Feichtenhofer et~al.(2022)Feichtenhofer, Li, He, et~al.]{feichtenhofer2022masked}
Christoph Feichtenhofer, Yanghao Li, Kaiming He, et~al.
\newblock Masked autoencoders as spatiotemporal learners.
\newblock \emph{NeruIPS}, 35:\penalty0 35946--35958, 2022.

\bibitem[Fu et~al.(2025)Fu, Lian, Wang, Shi, Wang, Yala, Darrell, Efros, and Goldberg]{fu2024rethinking}
Letian Fu, Long Lian, Renhao Wang, Baifeng Shi, XuDong Wang, Adam Yala, Trevor Darrell, Alexei~A Efros, and Ken Goldberg.
\newblock Rethinking patch dependence for masked autoencoders.
\newblock \emph{Transactions on Machine Learning Research}, 2025.

\bibitem[Hancock et~al.(1992)Hancock, Baddeley, and Smith]{pca_images}
Peter J.~B. Hancock, Roland~J. Baddeley, and Leslie~S. Smith.
\newblock The principal components of natural images.
\newblock \emph{Network: Computation In Neural Systems}, 3:\penalty0 61--70, 1992.

\bibitem[He et~al.(2022)He, Chen, Xie, Li, Doll'ar, and Girshick]{He2021MAE}
Kaiming He, Xinlei Chen, Saining Xie, Yanghao Li, Piotr Doll'ar, and Ross~B. Girshick.
\newblock Masked autoencoders are scalable vision learners.
\newblock \emph{CVPR}, 2022.

\bibitem[Hendrycks and Gimpel(2017)]{hendrycks2016gaussian}
Dan Hendrycks and Kevin Gimpel.
\newblock Bridging nonlinearities and stochastic regularizers with gaussian error linear units.
\newblock In \emph{ICLR}, 2017.

\bibitem[Hondru et~al.(2024)Hondru, Croitoru, Minaee, Ionescu, and Sebe]{hondru2024maskedimagemodelingsurvey}
Vlad Hondru, Florinel~Alin Croitoru, Shervin Minaee, Radu~Tudor Ionescu, and Nicu Sebe.
\newblock Masked image modeling: A survey, 2024.

\bibitem[Huang(2009)]{huang2009introduction}
Kerson Huang.
\newblock \emph{Introduction to statistical physics}.
\newblock Chapman and Hall/CRC, 2009.

\bibitem[Huang et~al.(2025)Huang, Wen, Chi, and Liang]{huang_how_2024}
Yu Huang, Zixin Wen, Yuejie Chi, and Yingbin Liang.
\newblock A theoretical analysis of self-supervised learning for vision transformers.
\newblock \emph{ICLR}, 2025.

\bibitem[Hyv{\"a}rinen et~al.(2009)Hyv{\"a}rinen, Hurri, and Hoyer]{natural_stats}
Aapo Hyv{\"a}rinen, Jarmo Hurri, and Patrik~O. Hoyer.
\newblock Natural image statistics - a probabilistic approach to early computational vision.
\newblock In \emph{Computational Imaging and Vision}, 2009.

\bibitem[Kadkhodaie and Simoncelli(2021)]{Kadkhodaie2021StochasticSF}
Zahra Kadkhodaie and Eero~P. Simoncelli.
\newblock Stochastic solutions for linear inverse problems using the prior implicit in a denoiser.
\newblock In \emph{NeurIPS}, 2021.

\bibitem[Kadkhodaie et~al.(2024)Kadkhodaie, Guth, Simoncelli, and Mallat]{Kadkhodaie2023GeneralizationID}
Zahra Kadkhodaie, Florentin Guth, Eero~P. Simoncelli, and St{'e}phane Mallat.
\newblock Generalization in diffusion models arises from geometry-adaptive harmonic representation.
\newblock \emph{ICLR}, 2024.

\bibitem[Kirillov et~al.(2023)Kirillov, Mintun, Ravi, Mao, Rolland, Gustafson, Xiao, Whitehead, Berg, Lo, Doll{\'a}r, and Girshick]{Kirillov2023SegmentA}
Alexander Kirillov, Eric Mintun, Nikhila Ravi, Hanzi Mao, Chloe Rolland, Laura Gustafson, Tete Xiao, Spencer Whitehead, Alexander~C. Berg, Wan-Yen Lo, Piotr Doll{\'a}r, and Ross~B. Girshick.
\newblock Segment anything.
\newblock \emph{ICCV}, 2023.

\bibitem[Kong et~al.(2023)Kong, Ma, Chen, Xing, Chi, Morency, and Zhang]{Kong2023UnderstandingMA}
Lingjing Kong, Martin~Q. Ma, Guan-Hong Chen, Eric~P. Xing, Yuejie Chi, Louis-Philippe Morency, and Kun Zhang.
\newblock Understanding masked autoencoders via hierarchical latent variable models.
\newblock \emph{CVPR}, 2023.

\bibitem[Kornblith et~al.(2019)Kornblith, Norouzi, Lee, and Hinton]{kornblith2019similarity}
Simon Kornblith, Mohammad Norouzi, Honglak Lee, and Geoffrey Hinton.
\newblock Similarity of neural network representations revisited.
\newblock In \emph{ICML}, 2019.

\bibitem[Li et~al.(2022)Li, Zheng, Liu, Wang, Su, and Zheng]{li2022semmae}
Gang Li, Heliang Zheng, Daqing Liu, Chaoyue Wang, Bing Su, and Changwen Zheng.
\newblock Semmae: Semantic-guided masking for learning masked autoencoders.
\newblock \emph{NeurIPS}, 2022.

\bibitem[Littwin et~al.(2024)Littwin, Saremi, Advani, Thilak, Nakkiran, Huang, and Susskind]{littwin2024jepa}
Etai Littwin, Omid Saremi, Madhu Advani, Vimal Thilak, Preetum Nakkiran, Chen Huang, and Joshua Susskind.
\newblock How jepa avoids noisy features: The implicit bias of deep linear self distillation networks.
\newblock \emph{NeurIPS}, 2024.

\bibitem[Liu et~al.(2023)Liu, Jiang, Li, Guo, Hu, Jiang, and Ren]{liu2023devil}
Hao Liu, Xinghua Jiang, Xin Li, Antai Guo, Yiqing Hu, Deqiang Jiang, and Bo Ren.
\newblock The devil is in the frequency: Geminated gestalt autoencoder for self-supervised visual pre-training.
\newblock In \emph{AIAA}, pages 1649--1656, 2023.

\bibitem[Liu et~al.(2021)Liu, Lin, Cao, Hu, Wei, Zhang, Lin, and Guo]{liu2021swin}
Ze Liu, Yutong Lin, Yue Cao, Han Hu, Yixuan Wei, Zheng Zhang, Stephen Lin, and Baining Guo.
\newblock Swin transformer: Hierarchical vision transformer using shifted windows.
\newblock In \emph{ICCV}, 2021.

\bibitem[Lu et~al.(2023)Lu, Clark, Zellers, Mottaghi, and Kembhavi]{lu2022unified}
Jiasen Lu, Christopher Clark, Rowan Zellers, Roozbeh Mottaghi, and Aniruddha Kembhavi.
\newblock Unified-io: A unified model for vision, language, and multi-modal tasks.
\newblock In \emph{ICLR}, 2023.

\bibitem[Milanfar and Delbracio(2024)]{Milanfar2024DenoisingAP}
Peyman Milanfar and Mauricio Delbracio.
\newblock Denoising: A powerful building-block for imaging, inverse problems, and machine learning.
\newblock \emph{ArXiv}, abs/2409.06219, 2024.

\bibitem[Mohan et~al.(2020)Mohan, Kadkhodaie, Simoncelli, and Fernandez‐Granda]{Mohan2019RobustAI}
Sreyas Mohan, Zahra Kadkhodaie, Eero~P. Simoncelli, and Carlos Fernandez‐Granda.
\newblock Robust and interpretable blind image denoising via bias-free convolutional neural networks.
\newblock \emph{ICLR}, 2020.

\bibitem[Park et~al.(2023)Park, Kim, Heo, Kim, and Yun]{park_what_2023}
Namuk Park, Wonjae Kim, Byeongho Heo, Taekyung Kim, and Sangdoo Yun.
\newblock What {Do} {Self}-{Supervised} {Vision} {Transformers} {Learn}?
\newblock In \emph{ICLR}, 2023.

\bibitem[Raghu et~al.(2021)Raghu, Unterthiner, Kornblith, Zhang, and Dosovitskiy]{Raghu2021DoVT}
Maithra Raghu, Thomas Unterthiner, Simon Kornblith, Chiyuan Zhang, and Alexey Dosovitskiy.
\newblock Do vision transformers see like convolutional neural networks?
\newblock In \emph{NeurIPS}, 2021.

\bibitem[Ramesh et~al.(2025)Ramesh, Bisulco, DiTullio, Wei, Balasubramanian, Daniilidis, and Chaudhari]{ramesh2024many}
Rahul Ramesh, Anthony Bisulco, Ronald~W DiTullio, Linran Wei, Vijay Balasubramanian, Kostas Daniilidis, and Pratik Chaudhari.
\newblock Many perception tasks are highly redundant functions of their input data.
\newblock \emph{COSYNE}, 2025.

\bibitem[Shekhar et~al.(2022)Shekhar, Bordes, Vincent, and Morcos]{Shekhar2022UnderstandingCV}
Shashank Shekhar, Florian Bordes, Pascal Vincent, and Ari~S. Morcos.
\newblock Understanding contrastive versus reconstructive self-supervised learning of vision transformers.
\newblock In \emph{NeurIPS}, 2022.

\bibitem[Shi et~al.(2023)Shi, Huang, Li, Zhang, Cheung, See, Qin, Dai, and Li]{shi2023flowformer++}
Xiaoyu Shi, Zhaoyang Huang, Dasong Li, Manyuan Zhang, Ka~Chun Cheung, Simon See, Hongwei Qin, Jifeng Dai, and Hongsheng Li.
\newblock Flowformer++: Masked cost volume autoencoding for pretraining optical flow estimation.
\newblock In \emph{CVPR}, 2023.

\bibitem[Srivastava et~al.(2014)Srivastava, Hinton, Krizhevsky, Sutskever, and Salakhutdinov]{Srivastava2014DropoutAS}
Nitish Srivastava, Geoffrey~E. Hinton, Alex Krizhevsky, Ilya Sutskever, and Ruslan Salakhutdinov.
\newblock Dropout: a simple way to prevent neural networks from overfitting.
\newblock \emph{J. Mach. Learn. Res.}, 15:\penalty0 1929--1958, 2014.

\bibitem[Steck(2020)]{Steck2020AutoencodersTD}
Harald Steck.
\newblock Autoencoders that don't overfit towards the identity.
\newblock In \emph{NeurIPS}, 2020.

\bibitem[Steiner et~al.(2022)Steiner, Kolesnikov, Zhai, Wightman, Uszkoreit, and Beyer]{steiner2021train}
Andreas Steiner, Alexander Kolesnikov, Xiaohua Zhai, Ross Wightman, Jakob Uszkoreit, and Lucas Beyer.
\newblock How to train your vit? data, augmentation, and regularization in vision transformers.
\newblock \emph{Transactions on Machine Learning Research}, 2022.

\bibitem[Tong et~al.(2022)Tong, Song, Wang, and Wang]{tong2022videomae}
Zhan Tong, Yibing Song, Jue Wang, and Limin Wang.
\newblock Videomae: Masked autoencoders are data-efficient learners for self-supervised video pre-training.
\newblock \emph{NeurIPS}, 2022.

\bibitem[Vincent et~al.(2010)Vincent, Larochelle, Lajoie, Bengio, Manzagol, and Bottou]{vincent2010stacked}
Pascal Vincent, Hugo Larochelle, Isabelle Lajoie, Yoshua Bengio, Pierre-Antoine Manzagol, and L{\'e}on Bottou.
\newblock Stacked denoising autoencoders: Learning useful representations in a deep network with a local denoising criterion.
\newblock \emph{Journal of machine learning research}, 11\penalty0 (12), 2010.

\bibitem[Wang et~al.(2023)Wang, Huang, Zhao, Tong, He, Wang, Wang, and Qiao]{Wang2023VideoMAEVS}
Limin Wang, Bingkun Huang, Zhiyu Zhao, Zhan Tong, Yinan He, Yi Wang, Yali Wang, and Yu Qiao.
\newblock Videomae v2: Scaling video masked autoencoders with dual masking.
\newblock \emph{CVPR}, 2023.

\bibitem[Wang et~al.(2024)Wang, Li, Li, He, Huang, Zhao, Zhang, Xu, Liu, Wang, et~al.]{wang2024internvideo}
Yi Wang, Kunchang Li, Yizhuo Li, Yinan He, Bingkun Huang, Zhiyu Zhao, Hongjie Zhang, Jilan Xu, Yi Liu, Zun Wang, et~al.
\newblock Internvideo: General video foundation models via generative and discriminative learning.
\newblock \emph{ECCV 2024}, 2024.

\bibitem[Wei et~al.(2022)Wei, Fan, Xie, Wu, Yuille, and Feichtenhofer]{Wei2021MaskedFP}
Chen Wei, Haoqi Fan, Saining Xie, Chaoxia Wu, Alan~Loddon Yuille, and Christoph Feichtenhofer.
\newblock Masked feature prediction for self-supervised visual pre-training.
\newblock \emph{CVPR}, 2022.

\bibitem[Wright and Ma(2022)]{Wright2022HighDimensionalDA}
John Wright and Y. Ma.
\newblock High-dimensional data analysis with low-dimensional models.
\newblock 2022.

\bibitem[Xie et~al.(2023{\natexlab{a}})Xie, Li, Zhan, Liu, Ong, and Loy]{xie2023masked}
Jiahao Xie, Wei Li, Xiaohang Zhan, Ziwei Liu, Yew~Soon Ong, and Chen~Change Loy.
\newblock Masked frequency modeling for self-supervised visual pre-training.
\newblock In \emph{ICLR}, 2023{\natexlab{a}}.

\bibitem[Xie et~al.(2024)Xie, Lee, Chen, and Finn]{xie2024self}
Johnathan Xie, Yoonho Lee, Annie~S Chen, and Chelsea Finn.
\newblock Self-guided masked autoencoders for domain-agnostic self-supervised learning.
\newblock \emph{ICLR}, 2024.

\bibitem[Xie et~al.(2022)Xie, Zhang, Cao, Lin, Bao, Yao, Dai, and Hu]{Xie2021SimMIMAS}
Zhenda Xie, Zheng Zhang, Yue Cao, Yutong Lin, Jianmin Bao, Zhuliang Yao, Qi Dai, and Han Hu.
\newblock Simmim: a simple framework for masked image modeling.
\newblock \emph{CVPR}, 2022.

\bibitem[Xie et~al.(2023{\natexlab{b}})Xie, Geng, Hu, Zhang, Hu, and Cao]{xie_revealing_2023}
Zhenda Xie, Zigang Geng, Jingcheng Hu, Zheng Zhang, Han Hu, and Yue Cao.
\newblock Revealing the {Dark} {Secrets} of {Masked} {Image} {Modeling}.
\newblock In \emph{CVPR}, 2023{\natexlab{b}}.

\bibitem[Xiong et~al.(2020)Xiong, Yang, He, Zheng, Zheng, Xing, Zhang, Lan, Wang, and Liu]{xiong2020layer}
Ruibin Xiong, Yunchang Yang, Di He, Kai Zheng, Shuxin Zheng, Chen Xing, Huishuai Zhang, Yanyan Lan, Liwei Wang, and Tieyan Liu.
\newblock On layer normalization in the transformer architecture.
\newblock In \emph{ICML}, 2020.

\bibitem[Zhang et~al.(2022)Zhang, Wang, and Wang]{Zhang2022HowMM}
Qi Zhang, Yifei Wang, and Yisen Wang.
\newblock How mask matters: Towards theoretical understandings of masked autoencoders.
\newblock \emph{NeurIPS}, 2022.

\end{thebibliography}
}

\clearpage
\appendix
\maketitlesupplementary

\section{Linear MAEs: Characterizing critical points}
We present additional details for~\cref{thm:mae_minima} below. 

Consider the loss. 
\begin{equation}
    \ell = ||X - (1-m) XAB||^2 + m(1-m) || G AB||^2,
\end{equation}
We first set $\partial \ell/\partial A$ to 0 to get
\begin{multline}
-2(1-m) X^\top  X B^\top   + 2(1-m)^2 X^\top  XABB^\top  + \\ 2m(1-m) \mathrm{Blkdiag}_p(X^\top  X) ABB^\top  = 0. 
\end{multline}
Any critical point must satisfy
\begin{equation}
    A^* = V^{-1} X^\top  X B^\top  (BB^\top )^{-1}
\end{equation}
where $V = (1-m)X^\top  X + m \mathrm{Blkdiag}_p(X^\top  X)$. Substituting this value of $A^*$ back into the loss, we get
\begin{align*}
\ell = & \Tr(X^\top  X) - \\
&\quad (1-m)\Tr[ B (X^\top  X  V^{-1} X^\top  X) B^\top  (BB^\top )^{-1} ]
\end{align*}
Let $C = X^\top  X  V^{-1} X^\top  X$ and $D=I$. 
Note that $C$ and $D$ are both symmetric and $D$ is invertible.
Using lemma~\ref{lemma:maximum}, the expression is minimized by the k largest eigenvalues of the generalized eigenvalue problem defined on $(C, D)$. Furthermore, every critical point is a subset of $k$ eigenvectors (from Lemma~\ref{lemma:generalized_eigenvalue}).

\section{Non-linear MAEs using linear approximations}\label{app:non_linear_mae}

\paragraph{Non-linear masked autoencoders under a Taylor series approximation.}
Consider a nonlinear autoencoder $f$ and the corresponding masked autoencoder loss
\[ \ell_m = \mathbb{E}_R || X - f(R \odot X) ||^2.\]
Let $X_{\mu} = (1-m)X$ and $X_r = R \odot X$. 
Under the Taylor series approximation around $0$
\[ f(X_R) \approx f(0) + X_R \nabla f(0)^\top   = X_R \nabla f(0)^\top   ,  \]
the MAE loss reduces to
\begin{align*}
    \ell_m  = ||X - (1-m) X \nabla f(0)^\top ||^2  + m(1-m) || G \nabla f(0)^\top  ||^2 
\end{align*}
Note that this approximation holds for small perturbations to the input, and is less likely to hold for large perturbations, i.e., the approximation is only valid for small masking ratio.

We will consider another approximation, but this time calculate the loss for a single sample. We consider a first-order approximation of $f$ for a single sample.
\[ f(x_R) \approx f(x_{\mu}) +  \nabla f(x_{\mu}) (x_R - x_\mu) ,  \]
which when substituted into the above loss gives us
\begin{align*}
    \ell_m(x)  = &||x - f(x_{\mu})||^2   \\
    &\qquad +m(1-m)  \Tr\left(F_x \mathrm{Blkdiag} (x^\top  x)   \right)
\end{align*}
Note that this is the loss for a single sample and $F_x$ is the Fisher information matrix for data point $x$. 

\paragraph{Masked autoencoders: A function space perspective}
Let us assume that we have access to the true input image signal $x(i, j)$, where $i, j \in [0, 1]$, as opposed to a discretized version of it. The masked autoencoder objective can be posed as an optimization problem over functionals $f \in \mathcal{F}$, i.e.,
\[ \ell_m = \int || f(r \odot x) - x||^2 \ \mathrm{d} r, \]
where $r$ is a mask applied to the image.
Assuming that $f$ is linear, then the above objective reduces to
\[ \ell_m =  ||x - f(\mu)||^2 - ||f(\mu)||^2 + \int || f(r \odot x)||^2 \ \mathrm{d}r, \]
where $\mu = \int (r \odot x) \ \mathrm{d}r$.
In the linear case, the MAE forces $f$ to reconstruct the mean masked image, while minimizing the variance of the predictions made on the masked images.

\section{Supporting lemmas}

\begin{lemma}\label{lemma:mae_expectation}
    The loss of the masked autoencoder is \[ \ell_m = ||X - (1-m)XAB||^2 + m(1-m) ||GAB||^2 \]
\end{lemma}
\begin{proof}
Expanding the term inside the expectation, we get
\begin{align*}
    \mathbb{E}_{R} ||X - &(R \odot X) AB ||^2 \\
    &= \mathbb{E}_{R} \text{Tr } [ X^\top X  - 2X^\top  (R \odot X) AB  \\
    &\qquad\quad + B^\top A^\top  (R \odot X)^\top  (R \odot X) A B ].
\end{align*}
We note that $\mathbb{E}[R \odot X] = (1-m)X$ and 
\begin{align*}
    \mathbb{E}[&(R \odot X)^\top  (R \odot X) ] =  \\
    &\qquad\begin{cases}
        (1-m)X_i^\top  X_j   & X_i, X_j \ \text{in the same patch} \\
        (1-m)^2X_i^\top X_j & X_i, X_j \ \text{not in the same patch}. 
    \end{cases} \\
    &= (1-m)^2X^\top  X + m(1-m) \mathrm{Blkdiag}_p (X^\top  X).
\end{align*}
Substituting back into the masked autoencoder loss, we get
\begin{align*}
    &\mathbb{E}_{R} ||X - (R \odot X) AB ||^2 \\
    &= \text{Tr } [ X^\top  X - 2(1-m) X^\top  XAB + (1-m)^2 X^\top  X  \\
    &\qquad + m(1-m) B^\top  A^\top  \mathrm{Blkdiag}_p (X^\top  X) AB ] \\
    &= ||X - (1-m)XAB||^2 + m(1-m) ||GAB||^2
\end{align*}
where $G^\top  G = \mathrm{Blkdiag}_p(X^\top  X)$.
\end{proof}

\begin{lemma}\label{lemma:generalized_eigenvalue}
    For matrices $C \in \mathbb{R}^{d \times d}$, $X \in \mathbb{R}^{d \times k}$ and an invertible matrix $D \in \mathbb{R}^{d \times d}$, every critical point of 
    \[ L(X) = \Tr [(X^\top  D X)^{-1} X^\top  C X] \]
    that is full-rank can be expressed as $U Q$, where $Q$ is an invertible matrix and $U$ is any subset of $k$ eigenvectors of the generalized eigenvalue problem for $(C, D)$.
\end{lemma}
\begin{proof}
Taking the derivative of $L(X)$ with respect to $X$ and setting it to 0, we get
\[ 2DX (X^\top  D X)^{-1} X^\top  C X = 2 C X. \]
Let $(\Lambda_D, \Phi_D)$ be the eigenvectors and eigenvalues of $D$.
Let $(\Lambda_{\bar C}, \Phi_{\bar C})$, be the eigenvectors and eigenvalues $\Lambda_D^{-1/2} \Phi_D^\top  C \Phi_D \Lambda_D^{-1/2}$.
In addition, we define $\Phi = \Phi_{\bar C} \Lambda_D^{-1/2} \Phi_{\bar D}$ and $\tilde X = \Phi X$.
We choose this definition of $\Phi$, since it diagonalizes both $C$ and $D$.
\begin{align*}
             &2DX (X^\top  D X)^{-1} X^\top  C X = 2 C X \\
    \implies &D\Phi^\top  \tilde X (\tilde X^\top  \Phi D \Phi^\top  \tilde X)^{-1} \tilde X^\top  \Phi C \Phi^\top  \tilde X =  C \Phi^\top  \tilde X \\
    \implies &D\Phi^\top  \tilde X (\tilde X^\top  \tilde X)^{-1} \tilde X^\top  \Lambda_{\tilde C} \tilde X =  C \Phi^\top  \tilde X \\
    \implies &\Phi D\Phi^\top  \tilde X (\tilde X^\top  \tilde X)^{-1} \tilde X^\top  \Lambda_{\tilde C} \tilde X =  \Phi C \Phi^\top  \tilde X \\
    \implies & \tilde X (\tilde X^\top  \tilde X)^{-1} \tilde X^\top  \Lambda_{\tilde C} \tilde X =  \Lambda_{\tilde C} \tilde X \\
    \implies & P_{\tilde X} \Lambda_{\tilde C} \tilde X =  \Lambda_{\tilde C} P_{\tilde X} \tilde X
\end{align*}
where $P_{\tilde X}$ is the projection operator.
Note that $P_{\tilde X} \tilde X = \tilde X$.
Since $\Lambda_{\bar C}$ is diagonal, $P_{\tilde X}$ must also be diagonal in order for the matrices to commute. Furthermore, $P_{\tilde X}$ has exactly $k$ eigenvalues equal to 1 and the rest set to 0, since $X$ has rank $k$. Hence, $\tilde X$ must be of the form $I_{S_k} Q$ where $S_k$ selects a subset of $k$ dimensions and $Q \in \mathbb{R}^{k \times k}$ is an invertible matrix. Hence $X = \Phi_{S_k} Q$ where $\Phi_{S_k}$ is a subset of $k$ eigenvectors of the generalized eigenvalue problem. 
\end{proof}
\begin{lemma}\label{lemma:maximum}
    For matrices $C \in \mathbb{R}^{d \times d}$, $X \in \mathbb{R}^{d \times k}$ and an invertible matrix $D \in \mathbb{R}^{d \times d}$, the global maximum of
    \[ L(X) = \Tr [(X^\top  D X)^{-1} X^\top  C X] \]
    is $\sum_{i=1}^k \Lambda_k$ where $\Lambda$ are the eigenvalues of the generalized eigenvalue problem $(C, D)$.
\end{lemma}
\begin{proof}
    From lemma~\ref{lemma:generalized_eigenvalue}, we know that any critical point is of the form $\Phi_{S_k} Q$. Subtituting this into $L(X)$, we get
    \begin{align*}
        L(X) 
        &= \Tr[ (X^\top  D X)^{-1} X^\top  C X ] \\
        &= \Tr[ (Q^\top  Q)^{-1} Q^\top  \Phi_{S_k}^\top  C \Phi_{S_k} Q ] \\
        &= \Tr[ \Phi_{S_k}^\top  C \Phi_{S_k}] = \sum_{i \in S_k} \Lambda_i.
    \end{align*}
    The loss is maximized by the largest $k$ eigenvalues and minimized by the smallest $k$ eigenvalues.

\end{proof}

\begin{lemma}\label{lemma:mae_taylor1}
Under the Taylor series approximation of 
$f(X_R) \approx f(0) + X_R \nabla f(0)^\top   = X_R \nabla f(0)^\top $,
the MAE loss for a non-linear function $f$ is 
\[  \ell_m = ||X - (1-m) X \nabla f(0)^\top ||^2  + m(1-m) || G \nabla f(0)^\top  ||^2.  \]
\end{lemma}
\begin{proof}
\begin{align*}
    \ell_m  &= \mathbb{E}_R ||X -  X_R \nabla f(0)^\top  ||^2 \\
    &= ||X||^2 
    + \mathbb{E}_R || X_R \nabla f(0)^\top  ||^2  
    -2\mathbb{E}_R \Tr( X^\top  X_R \nabla f(0)^\top ) \\
    &= \Tr( X^\top  X )  +
    (1-m)^2 \Tr \left( \nabla f(0) X^\top  X \nabla f(0)^\top  \right) \\
    &\qquad + m(1-m) \Tr \left( \nabla f(0)  \mathrm{Blkdiag}(X^\top  X) \nabla f(0)^\top  \right) \\
    &\qquad -2 (1-m) \mathbb{E}_R \left[ X^\top  X \nabla f(0)^\top  \right] \\
    &= ||X - (1-m) X \nabla f(0)^\top ||^2  + m(1-m) || G \nabla f(0)^\top  ||^2 .
\end{align*}
\end{proof}

\begin{lemma}\label{lemma:mae_taylor2}
Under the Taylor series approximation of 
$f(x_R) \approx f(x_{\mu}) +  \nabla f(x_{\mu}) (x_R - x_\mu)$,
the MAE loss, reduces to
\[ 
    \ell_m(x)  = ||x - f(x_{\mu})||^2 +  m(1-m)  \Tr\left(F_x \mathrm{Blkdiag} (x^\top  x)   \right).
\]
\end{lemma}
\begin{proof}
\begin{align*}
    \ell_m(x)  &= ||x - f(x_{\mu}) - \nabla f(x_{\mu}) (x_R - x_\mu)||^2 \\
    &= ||x - f(x_{\mu})||^2 +  \mathbb{E}_R ||\nabla f(x_{\mu}) (x_R - x_\mu)||^2  \\ 
    &\qquad +2\mathbb{E}_R (x - f(x_{\mu}))^\top  (\nabla f(x_{\mu}) (x_R - x_\mu))  \\
    &= ||x - f(x_{\mu})||^2 + \\
    &\qquad \quad \mathbb{E}_R (x_R - x_\mu)^\top  \nabla f(x_{\mu})^\top  \nabla f(x_{\mu}) (x_R - x_\mu)  \\
    &= ||x - f(x_{\mu})||^2 + \\
    &\qquad \mathbb{E}_R \Tr \left(\nabla f(x_{\mu})^\top  \nabla f(x_{\mu}) (x_R - x_\mu)(x_R - x_\mu)^\top    \right) \\
    &= ||x - f(x_{\mu})||^2 +  m(1-m)  \Tr\left(F_x \mathrm{Blkdiag} (x^\top  x)   \right).
\end{align*}
\end{proof}

\begin{lemma}\label{lemma:mae_functional}
The masked autoencoder loss, for a input image $x$ and linear functional $f$ is
\[ \ell_m =  ||x - f(\mu)||^2 - ||f(\mu)||^2 + \int || f(r \odot x)||^2 \ \mathrm{d}r, \]
where $\mu = \int (r \odot x) \ \mathrm{d}r $
\end{lemma}
\begin{proof}
\begin{align*}
    \ell_m 
    &= \int || x - f(r \odot x)||^2 \ \mathrm{d} r \\
    &= \int || x - f(\mu) - (f(r \odot x) - f(\mu))||^2 \ \mathrm{d} r \\
    &= \int ||f(\mu) - x||^2 + || f(r \odot x)- f(\mu)||^2  \ \mathrm{d} r \\
    &\qquad - \int 2 \langle x - f(\mu),  f(r \odot x)- f(\mu) \rangle \mathrm{d} r \\
    &= \int  ||f(\mu) - x||^2 + || f(r \odot x)- f(\mu)||^2 \ \mathrm{d} r \\
    &= || x - f(\mu)||^2 - ||f(\mu)||^2 + \int ||f(r \odot x)||^2 \ \mathrm{d} r.
\end{align*}
\end{proof}

\section{Additional experiments}\label{app:expts}
\begin{figure}
    \centering
    \includegraphics[width=\linewidth]{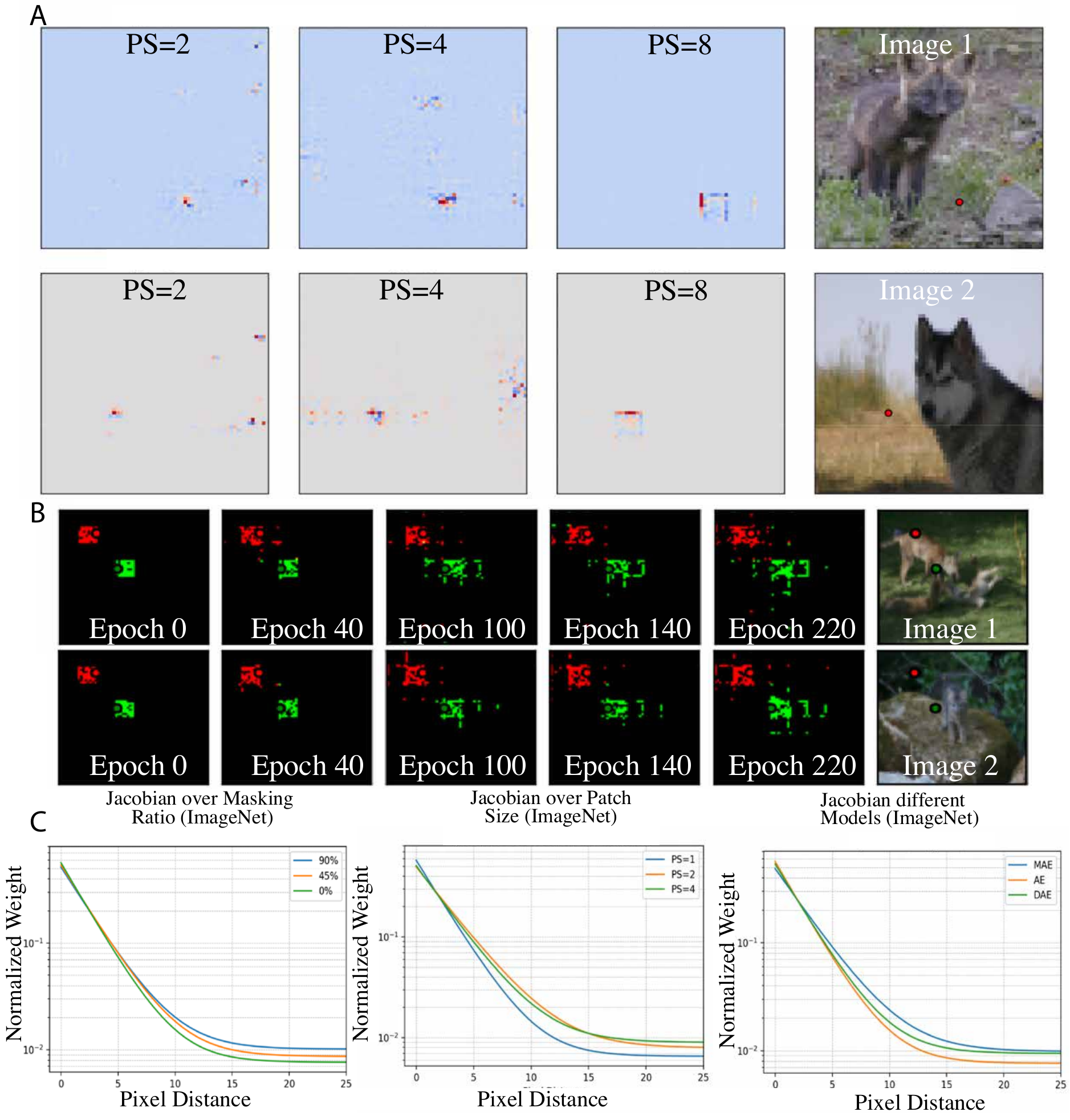}
    \caption{Analogous to Figs. 5, 6, 8, 9, but for the ImageNet-64 dataset. A) Visualization of the Jacobian ($m = 0.8$) for nonlinear MAE. B) Jacobian across different stages of training for a specific output pixel ($ps = 8$, $m = 0.8$) for nonlinear MAE. C) Normalized weight for different hyper-parameters of a linear MAE, AE and DAE.}
    \label{fig:imgnet-jacob}
\end{figure}
\paragraph{Additional details about MAE pretraining}

We train MAEs using the architecture in~\citet{He2021MAE}.
We divide the image into patches of size $p$ and randomly mask a fraction $m$ of the patches before feeding it to the MAE.
The encoder projects the unmasked patches to a $d$-dimensional embedding using a linear layer. The sequence of patches are then fed to a series of Transformer blocks.
The decoder adds a learnable vector and position encoding for every masked patch and reconstructs the masked patches. 

\begin{figure}[!htb]
    \centering
    \includegraphics[width=0.9\linewidth]{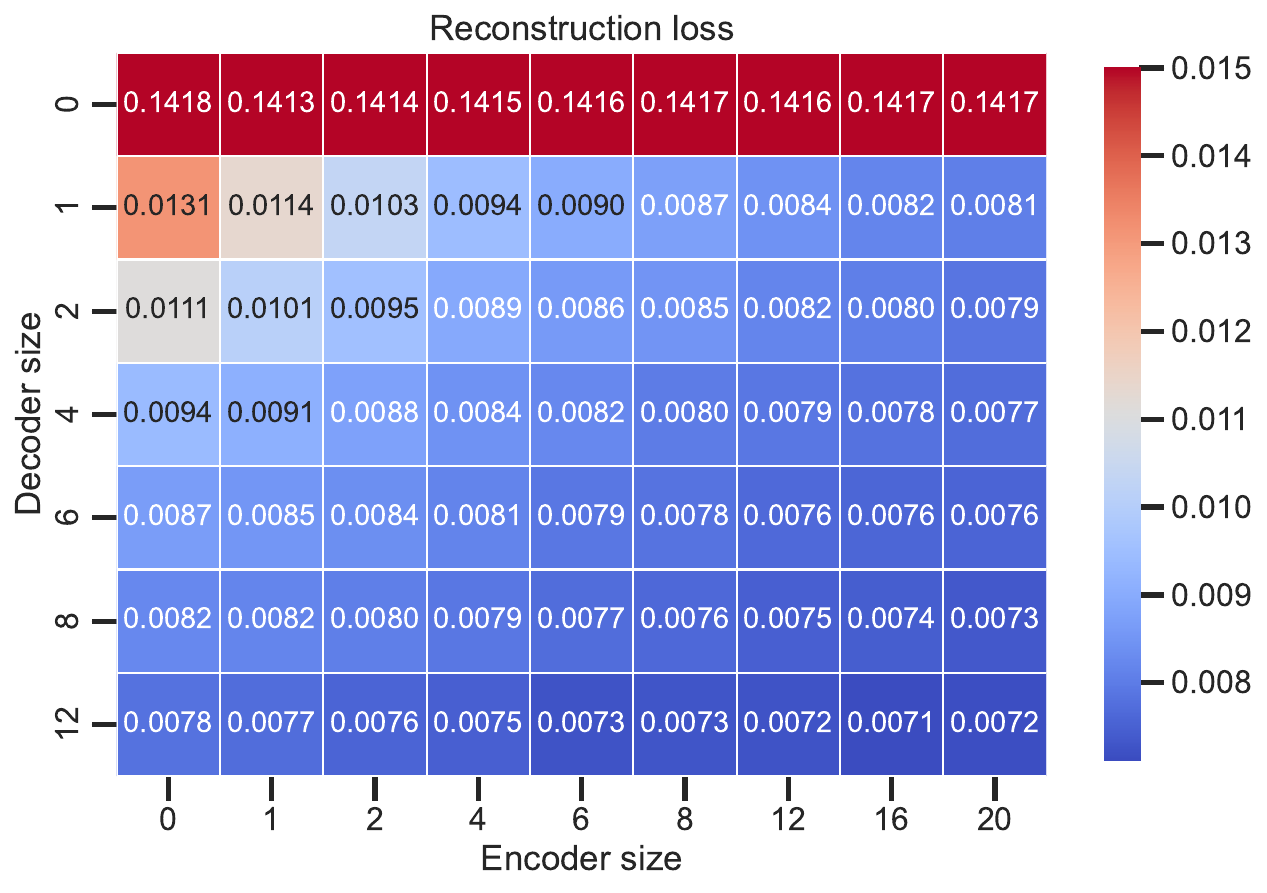}
    \includegraphics[width=0.9\linewidth]{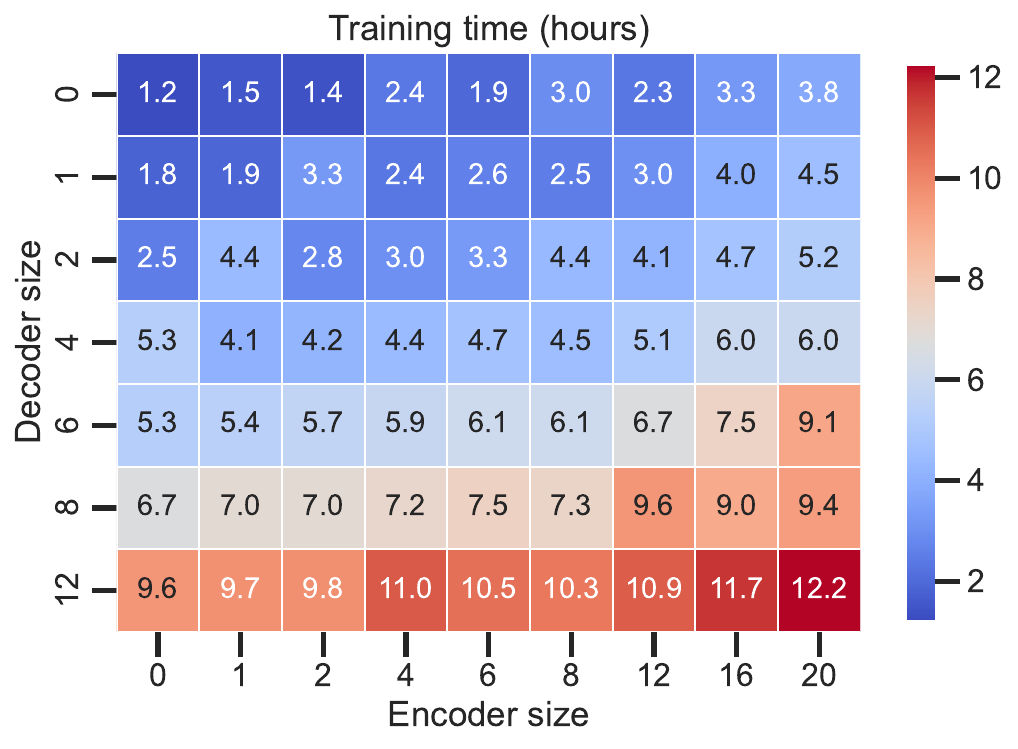}
    \caption{
        We vary the number of encoder and decoder layers and record the reconstruction loss at the end of training and the time required to train the MAEs. Note that \textbf{MAEs are slow to train} with training time growing faster the the size of the decoder.
    The reconstruction loss becomes smaller with increasing size of both the decoder and the encoder. However, we find that the \textbf{training loss is not a good proxy for downstream task performance}.}
    \label{fig:decoder_size_loss_time}
\end{figure}

\paragraph{Number of Encoder and decoder layers}
We train MAEs on CIFAR10 for different encoder and decoder sizes.
We find that the reconstruction loss decreases with increasing size of both the encoder and the decoder (\cref{fig:decoder_size_loss_time}). 
However, the training time grows faster than the size of the decoder, making it computationally expensive to train large decoders.
We also find that training loss is not a good proxy for downstream task performance. We evaluate the performance of the trained encoder using linear probing and find that the accuracy improves as we increase the size of the encoder. However, the optimal decoder size is 2-4 layers (\cref{fig:decoder_size_acc}).

If the model are trained only using the supervised loss, i.e., we do no MAE pretraining, then the accuracy on CIFAR plateaus around 83-84\%. In fact the accuracy for a 20-layer network is worse than the accuracy for a 12-layer network which differs from the trend for masked autoencoders.

\begin{figure}[!htb]
    \centering
    \includegraphics[width=0.89\linewidth]{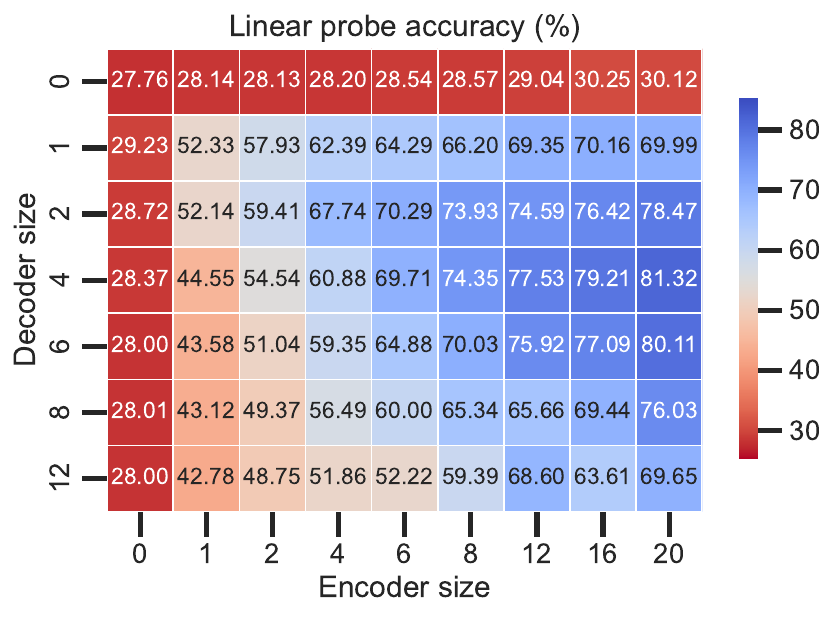}
    \includegraphics[width=0.89\linewidth]{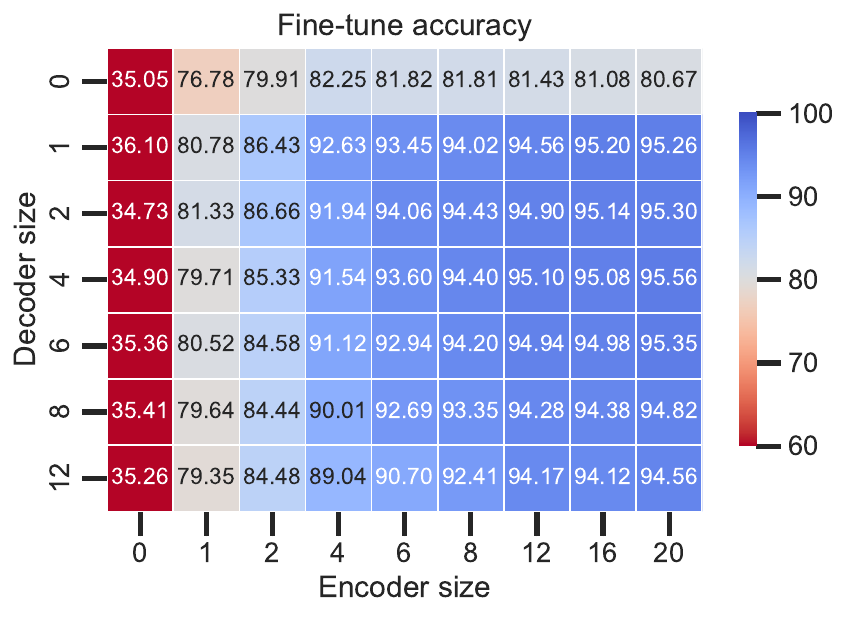}
    \caption{
        We vary the number of encoder and decoder layers (transformer blocks) and plot the linear probe accuracy (left) and the accuracy after fine-tuning for 100 epochs.
    The \textbf{accuracy of the trained encoder continues to improve as we increase its size}. Linear probe accuracies are usually indicative of performance after fine-tuning.}
    \label{fig:decoder_size_acc}
\end{figure}

\begin{figure}[!htb]
    \centering
    \includegraphics[width=0.48\linewidth]{imgs/expts/mp_loss.pdf}
    \includegraphics[width=0.49\linewidth]{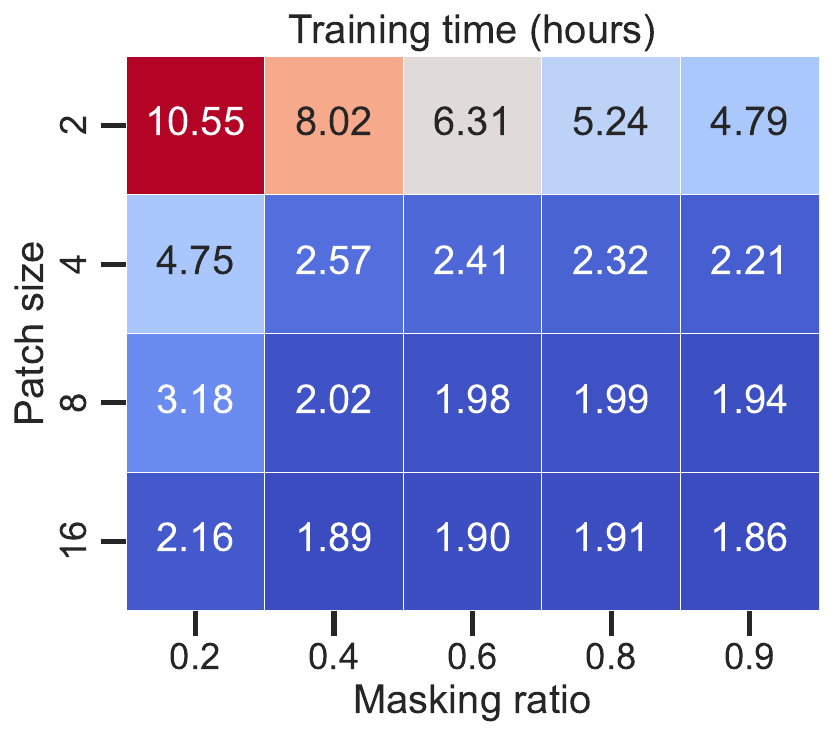}
    \caption{
        We vary the masking ratio and patch-size of the masked autoencoder. While larger masking ratio lead to smaller training times, even smaller masking ratios work quite well. Smaller patch-sizes are a lot slower to train but usually perform better than larger patch-sizes.
    }
\label{fig:mask_patch_loss_time}
\end{figure}

\begin{figure}[!htb]
    \centering
    \includegraphics[width=0.49\linewidth]{imgs/expts/mp_acc_linear.pdf}
    \includegraphics[width=0.49\linewidth]{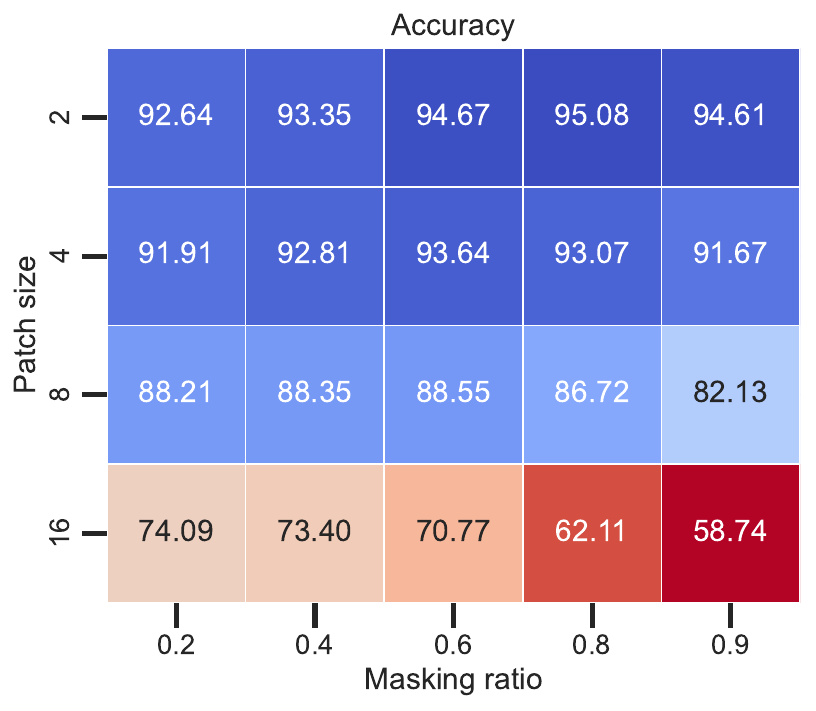}
    \caption{
        We plot the linear probe (Left) and accuracy after fine-tuning (right) for different patch-size and masking ratio. Patch-sizes of 2 and 4
    }
    \label{fig:mask_patch_acc}
\end{figure}

\paragraph{Patch-size vs. Masking ratio}
The masking ratio and patch-size control the basis learned by the MAE. We surprisingly find that many different parameters work surprisingly well for downstream task accuracy. We also note that reconstruction loss is not indicative of downstream task performance.

\paragraph{Are long training times even necessary?}
MAEs are typically trained for a large number of epochs and the reconstruction error continues to decrease over the course of training.
However, the reconstruction error is not predictive of both the linear-probe and fine-tuning accuracies.
In~\cref{fig:decoder_time}, we consider multiple checkpoints over the course of pretraining and plot the number of pretraining epochs against the linear probe accuracy of that checkpoint.
We find that the \textbf{linear probe accuracy continues to increase even after 1500 epochs of training} particularly for larger models, justifying the need to pretrain for a large number of epochs (see \cref{fig:decoder_size_acc}).
\begin{figure}[!htb]
    \centering
    \includegraphics[width=0.80\linewidth]{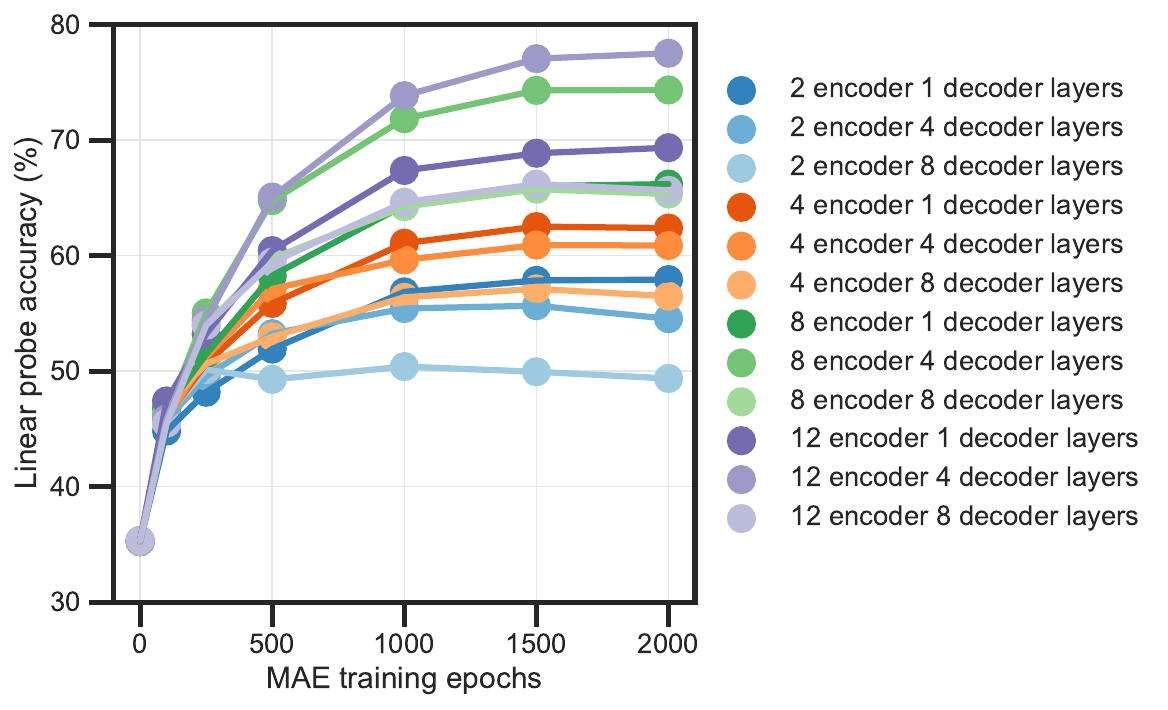}
    \caption{
        We plot the linear probe accuracies of different masked autoencoders over the course of training. They accuracy continue to increase even after 1000 epochs of training, justifying the need for long training times. Larger models tend to require longer training times.
    }
    \label{fig:decoder_ckpt}
\end{figure}

\paragraph{Centered kernel alignment} or
CKA~\citep{kornblith2019similarity} measures the similarity between representations of two different networks. We use CKA to measure the similarity between the representations of MAEs trained with different number of encoder layers and with 4 decoder layers. White indicates that the similarity is high and black/red indicates that the similarity is low. We find that the larger networks are more similar to the 12-layer encoder while while the smaller networks are less similar to the 12-layer network, particularly at the last layer.

\begin{figure}
    \centering
    \includegraphics[width=0.6\linewidth]{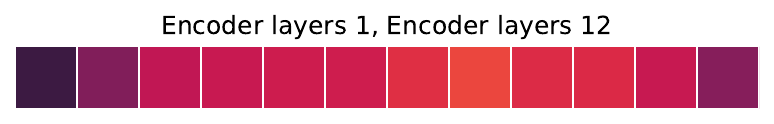}
    \includegraphics[width=0.6\linewidth]{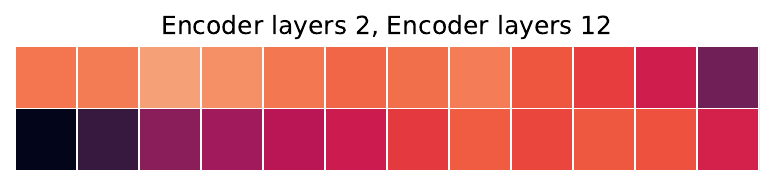}
    \includegraphics[width=0.6\linewidth]{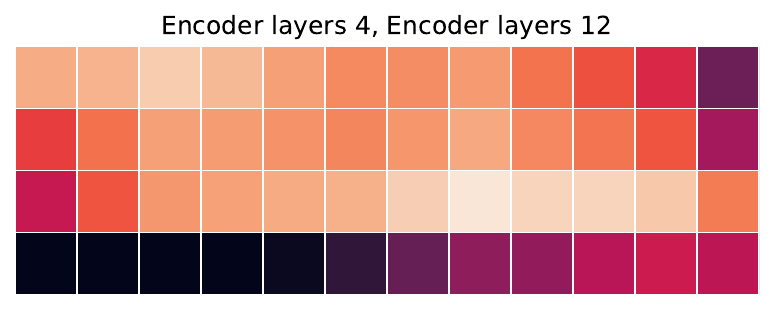}
    \includegraphics[width=0.6\linewidth]{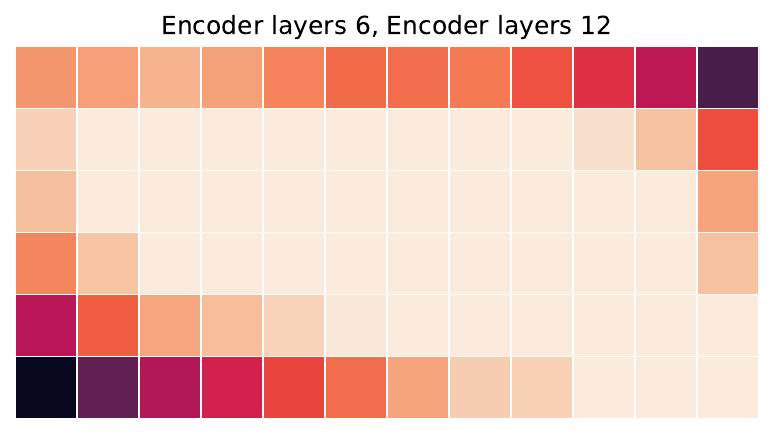}
    \caption{CKA between MAEs trained with different number of encoder layers. Each row and column corresponds to the similarity between a $k$-layer encoder and the 12-layer encoder (hence 12 columns). The darker shades indicate that the representations for those two layers are not similar.}
    \label{fig:cka}
\end{figure}

\section{More experiments with the Ising model}

\begin{figure}[H]
    \centering
    \includegraphics[width=0.48\linewidth]{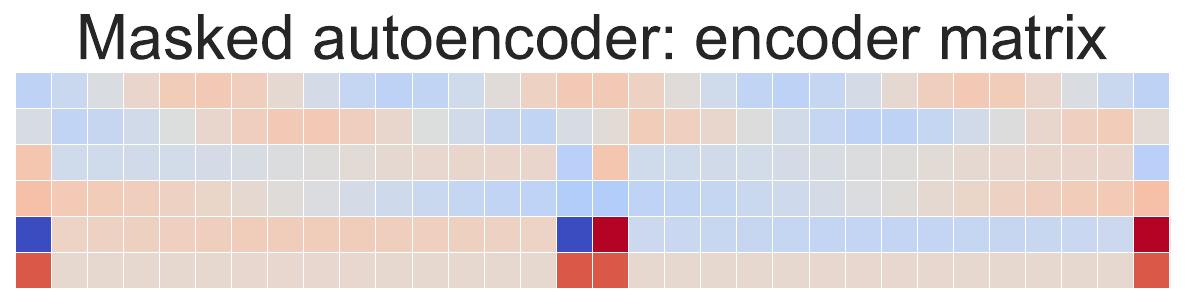}
    \includegraphics[width=0.48\linewidth]{imgs/linear_mae/A_p8_k6_m50.pdf}
    \includegraphics[width=0.48\linewidth]{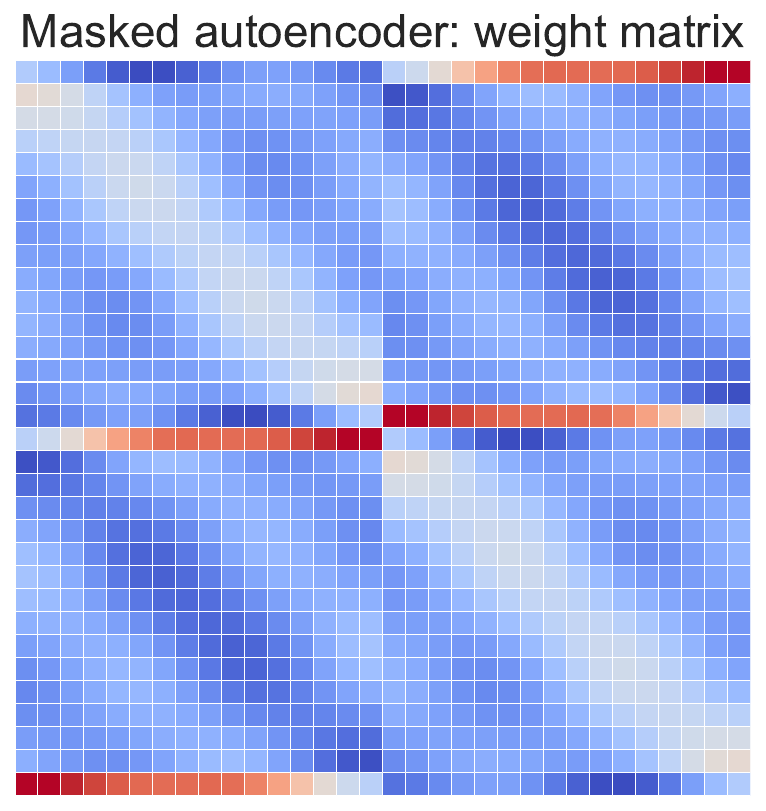}
    \includegraphics[width=0.48\linewidth]{imgs/linear_mae/W_p8_k6_m50.pdf}
    \caption{
        We plot the weight matrix $AB$ and the encoder matrix $A$ for \textbf{(left)} patch-size 16 and masking ratio of 0.5 and \textbf{(right)} patch-size 8 and masking ratio 0.5. 
        Increasing the patch-size while keeping the masking ratio fixed biases the encoder towards features that capture long-range correlations.
    }
    \label{fig:app_ising_mae1}
\end{figure}

\begin{figure}[H]
    \centering
    \includegraphics[width=0.48\linewidth]{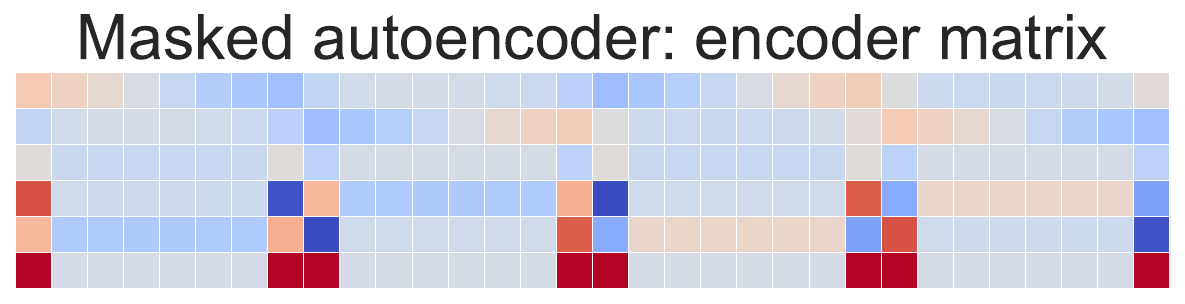}
    \includegraphics[width=0.48\linewidth]{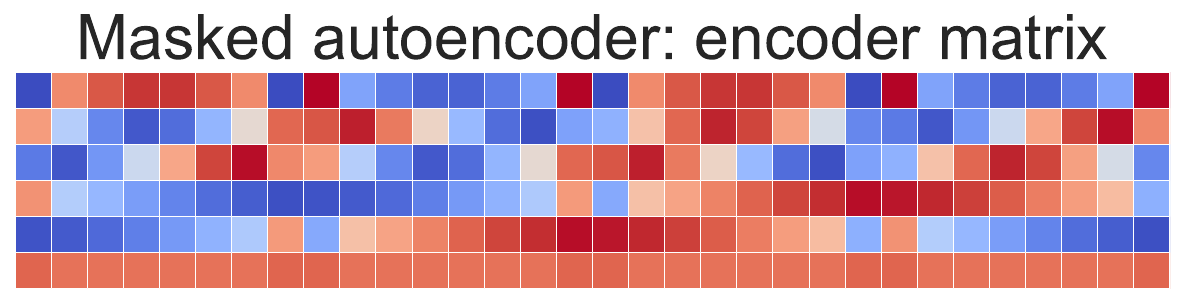}
    \includegraphics[width=0.48\linewidth]{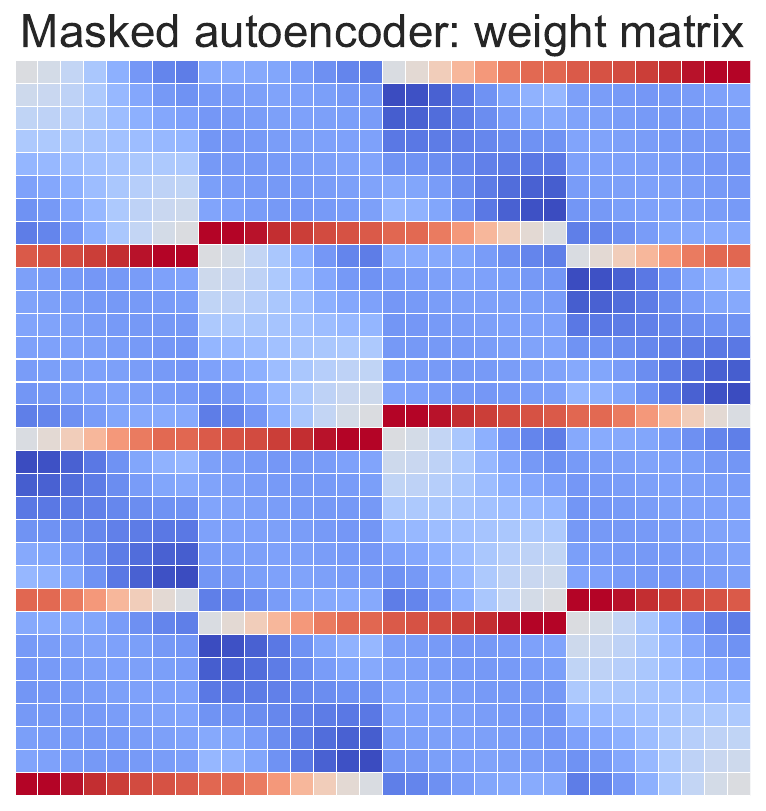}
    \includegraphics[width=0.48\linewidth]{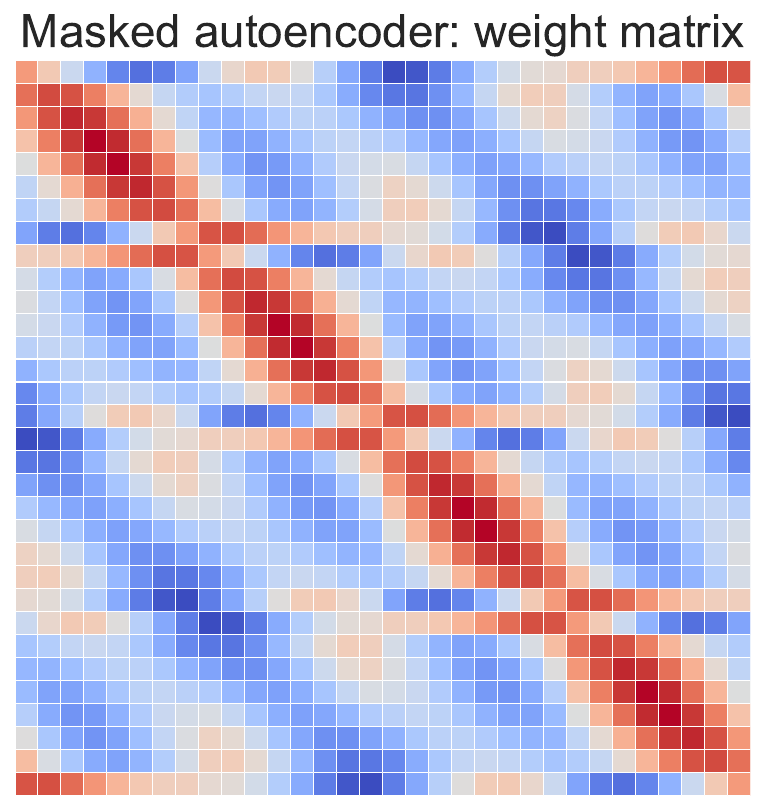}
    \caption{
        We plot the weight matrix $AB$ and the encoder matrix $A$ for \textbf{(left)} patch-size 16 and masking ratio of 0.99 and \textbf{(right)} patch-size 8 and masking ratio (0.01). 
        Reducing the masking ratio biases the encoder towards features based on local correlations while increasing the masking ratio prioritizes features that capture long range correlations.
    }
    \label{fig:app_ising_mae2}
\end{figure}

\section{MAEs in the Frequency Domain}\label{app:freq}
\begin{figure}[H]
    \centering
\includegraphics[width=\linewidth]{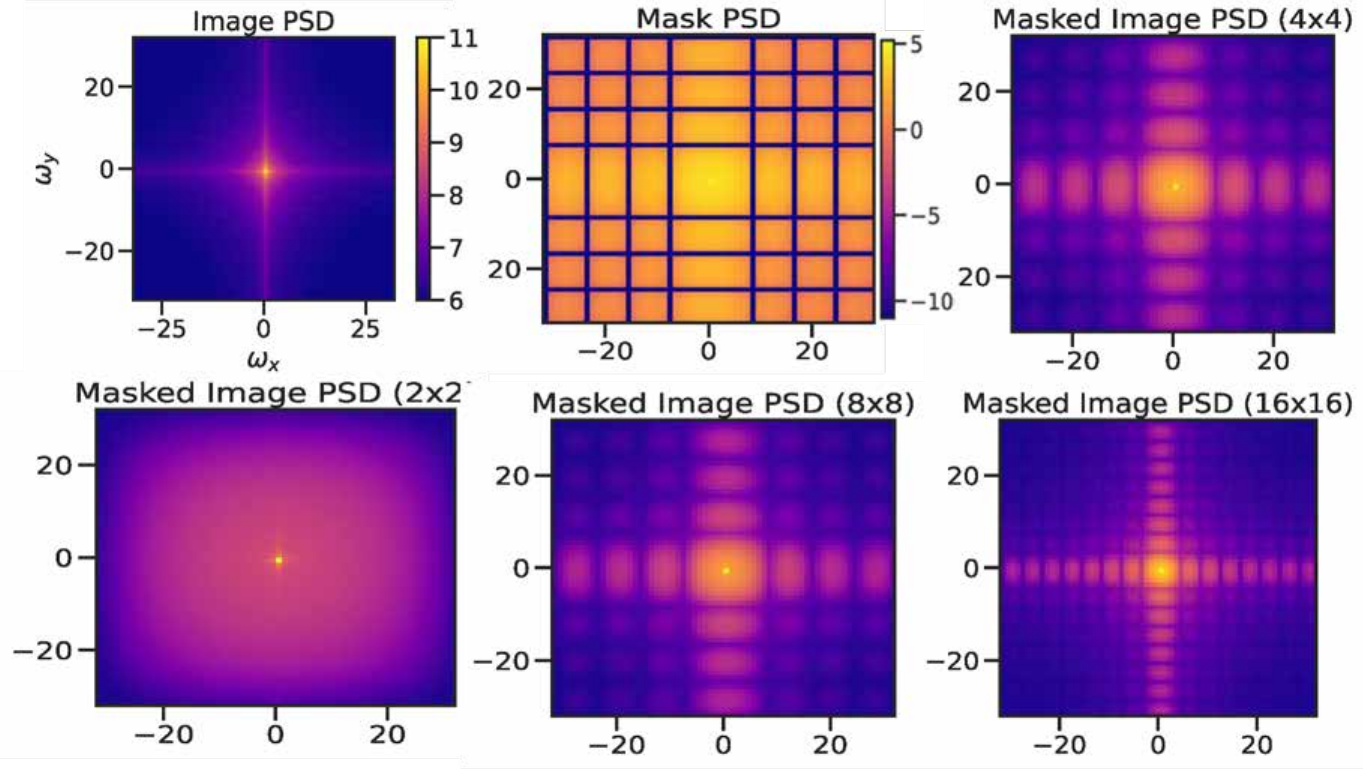}
    \caption{Masking using square patches in a MAE zeros out frequencies in a grating-like pattern in Fourier space~(color bar in image log power spectral density used for all plots, but mask). }
    \label{fig:mae-spec}
\end{figure}

MAEs mask a part of the input image (20\% of the patches), usually multiple square patches, and train an encoder-decoder pair to reconstruct the intensity at the masked patches.
\cref{fig:mae-spec} (top) shows the power spectral density (PSD), i.e., squared amplitude at different spatial frequencies, of the original image (left), the mask itself (top, middle) and masked images using different patch sizes.
The Discrete Fourier Transform~(DFT) magnitude of these masks is a sinc function modulated by a sum of complex exponentials~\cref{eq:fft_abs}. This can be shown in a 1D setting by considering a discrete signal $x \in \mathbb{R}^D$. For MAEs with patch size $p$, we parameterize each individual mask using a rectangular pulse defined by function $r$ as 

\begin{equation}
r[n] = u[n] - u[n - p],
\end{equation}

where $p$ is the patch size, n is the discrete time index and $u[n]$ is the Heaviside step function. A MAE mask $m \in \mathbb{R}^D$ is constructed as a sum of such rectangular pulses:

\begin{equation}
m[n] = \sum_{i=1}^N r[n-a_i],
\end{equation}
where $a_i$ denotes the starting index of the ith mask. 
The masked signal is then given by elementwise multiplication of the signal $x$ and the mask $m$:

\begin{equation}
y[n] = x[n] \cdot m[n].
\end{equation}

In the frequency domain, by the multiplication property of the DFT, this becomes:

\begin{equation}
Y[k] = \frac{1}{D} X[k] * M[k],
\end{equation}

where $X[k]$ and $M[k]$ are the DFTs of $x[n]$ and $m[n]$, respectively and $k\in [0,1,...,N-1]$ is the frequency index.

Using the time-shift and linearity properties of the DFT, the transform of the mask~$m[n]$ can be written as:

\begin{equation}
M[k] = \sum_{i=1}^N e^{-j 2\pi k a_i} \cdot R[k],
\end{equation}

where the DFT of $r[n]$ is given by:

\begin{equation}
R[k] = e^{-j \pi k (p - 1)/D}  \frac{\sin(\pi pk/D)}{\sin(\pi k/D)}.
\end{equation}

Hence, applying MAE-style masking results in a spectral smoothing effect analogous to sinc filtering, as illustrated in \cref{fig:mae-spec}.
\begin{equation}
\left|R[k]\right| = \left|\frac{\sin(\pi pk/D)}{\sin(\pi k/D)}\right|  \left|\sum_{i=1}^N e^{-j 2\pi k a_i}\right|. \label{eq:fft_abs}
\end{equation}

A square patch therefore corresponds to masking frequencies in a grating-like pattern. Masking in an MAE therefore corresponds to zeroing out certain frequencies. The goal of an MAE is to interpolate the value of the masked frequencies from the unmasked version of the spectrum.  

\section{Related work}\label{s:app_related}
\paragraph{Masked Autoencoders.} The goal of visual representation learning is to develop representations that go beyond the identity function~\cite{Steck2020AutoencodersTD} and instead are predictive of downstream visual tasks. Early methods in this field sought to prevent trivial identity mappings by incorporating various noise distributions, such as blankout noise in Dropout~\cite{Srivastava2014DropoutAS} or Marginalized Autoencoders~\cite{Chen2015MarginalizingSL}, and Gaussian noise in Denoising Autoencoders~\cite{vincent2010stacked}. A key limitation of these earlier approaches is that they were developed before deep networks could be easily trained end-to-end, resulting in reliance on greedy layer-wise training, which does not generalize as well as full end-to-end optimization.
Modern methods, such as Data2Vec~\cite{baevski2022data2vec} and BERT~\cite{Devlin2019BERTPO}, overcome this limitation by training Vision Transformers (ViTs) end-to-end using blackout noise, achieving state-of-the-art representation performance. The core innovation of modern MAEs~\cite{He2021MAE, Xie2021SimMIMAS} lies in their architectural design, which encodes only masked image patches, significantly reducing computational costs for masked prediction strategies. Recent advancements in MAEs include cross-attention mechanisms for efficient computation~\cite{fu2024rethinking}, frequency-based loss functions to mitigate oversmoothing~\cite{xie2023masked, liu2023devil}, and intermediate perception reconstruction instead of direct output reconstruction~\cite{Wei2021MaskedFP, shi2023flowformer++}. Additional innovations involve task-guided masking strategies~\cite{li2022semmae, Fan2023MotionGuidedMF, xie2024self}, distillation-based training~\cite{Bai2022MaskedAE}, and numerous other techniques~\cite{hondru2024maskedimagemodelingsurvey}.\\

\end{document}